\pdfoutput=1
\documentclass[letterpaper]{article}
\usepackage[totalwidth=480pt, totalheight=680pt]{geometry}

\usepackage{algorithm}
\usepackage{algorithmicx}
\usepackage{algpseudocode}
\usepackage{subcaption}

\usepackage{enumerate}

\usepackage{amsmath}
\usepackage{amssymb}
\usepackage{amsthm}
\usepackage{pifont}
\usepackage{booktabs}
\usepackage{xcolor}
\definecolor{darkgreen}{rgb}{0,0.5,0}
\usepackage{hyperref}
\hypersetup{
    unicode=false,          
    colorlinks=true,        
    linkcolor=red,          
    citecolor=darkgreen,    
    filecolor=magenta,      
    urlcolor=blue           
}
\usepackage[capitalize, nameinlink]{cleveref}

\usepackage{thm-restate}
\theoremstyle{plain}
\newtheorem{theorem}{Theorem}

\newtheorem{lemma}[theorem]{Lemma}

\theoremstyle{definition}
\newtheorem{definition}[theorem]{Definition}

\theoremstyle{remark}

\newtheorem*{remark*}{Remark}

\usepackage{tikz}
\usetikzlibrary{calc, graphs, graphs.standard, shapes, arrows, arrows.meta, positioning, decorations.pathreplacing, decorations.markings, decorations.pathmorphing, fit, matrix, patterns, shapes.misc, tikzmark}

\newcommand{\R}{\mathbb{R}}
\newcommand{\cA}{\mathcal{A}}
\newcommand{\cB}{\mathcal{B}}

\newcommand{\cE}{\mathcal{E}}
\newcommand{\cI}{\mathcal{I}}
\newcommand{\cJ}{\mathcal{J}}
\newcommand{\cO}{\mathcal{O}}
\newcommand{\cX}{\mathcal{X}}

\newcommand{\opt}{\mathrm{\textbf{opt}}}
\newcommand{\argmax}{\mathrm{argmax}}
\newcommand{\argmin}{\mathrm{argmin}}

\newcommand{\Des}{\texttt{Des}}
\newcommand{\Anc}{\texttt{Anc}}
\newcommand{\Pa}{\texttt{Pa}}
\newcommand{\Ch}{\texttt{Ch}}
\newcommand{\last}{\texttt{last}}

\title{Subset verification and search algorithms for causal DAGs}
\author{
Davin Choo\thanks{Equal contribution}\\
National University of Singapore\\
\texttt{davin@u.nus.edu}
\and
Kirankumar Shiragur\footnotemark[1]\\
Broad Institute of MIT and Harvard\\
\texttt{shiragur@stanford.edu}
}
\date{}

\begin{document}

\maketitle

\begin{abstract}
Learning causal relationships between variables is a fundamental task in causal inference and directed acyclic graphs (DAGs) are a popular choice to represent the causal relationships.
As one can recover a causal graph only up to its Markov equivalence class from observations, interventions are often used for the recovery task.
Interventions are costly in general and it is important to design algorithms that minimize the number of interventions performed.

In this work, we study the problem of identifying the smallest set of interventions required to learn the causal relationships between a subset of edges (target edges).
Under the assumptions of faithfulness, causal sufficiency, and ideal interventions, we study this problem in two settings: when the underlying ground truth causal graph is known (subset verification) and when it is unknown (subset search).

For the subset verification problem, we provide an efficient algorithm to compute a minimum sized interventional set; we further extend these results to bounded size non-atomic interventions and node-dependent interventional costs.
For the subset search problem, in the worst case, we show that no algorithm (even with adaptivity or randomization) can achieve an approximation ratio that is asymptotically better than the vertex cover of the target edges when compared with the subset verification number.
This result is surprising as there exists a logarithmic approximation algorithm for the search problem when we wish to recover the whole causal graph.
To obtain our results, we prove several interesting structural properties of interventional causal graphs that we believe have applications beyond the subset verification/search problems studied here.
\end{abstract}

\section{Introduction}
\label{sec:introduction}

Learning causal relationships from data is an important problem with applications in many areas such as biology \cite{king2004functional,sverchkov2017review,rotmensch2017learning,pingault2018using}, economics \cite{hoover1990logic,rubin2006estimating}, and philosopy \cite{reichenbach1956direction,woodward2005making,eberhardt2007interventions}.
More recently, causal inference techniques have also been used to design methods that generalize to out-of-distribution samples \cite{ganin2016domain,arjovsky2019invariant,arjovsky2020out}.
In many of these applications, directed acyclic graphs (DAG) are used to model the causal relationships, where an arc $x \to y$ encodes $x$ causes $y$, and the goal is to recover these graphs from data.

One can only recover causal graphs up to a Markov equivalence class using observational data \cite{pearl2009causality,spirtes2000causation}, and additional model assumptions \cite{shimizu2006linear,peters2014identifiability,mooij2016distinguishing} or interventions \cite{eberhardt2006n,eberhardt2010causal,eberhardt2012number,hu2014randomized,shanmugam2015learning,greenewald2019sample,squires2020active} are often used to recover the true underlying causal graph.
In this work, we study the causal discovery problem using interventions.
Performing interventions in real life are often costly as in many cases they correspond to performing randomized controlled trials or gene knockout experiments.
As such, most of the prior works focused on recovering the causal graph while minimizing the interventions performed.

Besides minimizing the number of interventions performed, many applications care about recovering only a \emph{subset} of the causal relationships.
For instance, in \emph{local graph discovery}, efficient learning of localized causal relationships play a central role in feature selection via Markov blankets\footnote{A Markov blanket of a variable $X \in V$ is a subset of variables $S \subseteq V$ such that all other variables are independent of $X$, conditioned on $S$.} \cite{aliferis2003hiton,tsamardinos2003towards,mani2004causal,aliferis2010local} while scalability is of significant concern when one only wishes to learn localized causal effects (e.g.\ the direct causes and effects of a target variable of interest) \cite{statnikov2015ultra,frangieh2021multimodal} within a potentially large causal graph (e.g.\ gene regulatory networks~\cite{davidson2005gene}).
Meanwhile, in the context of designing algorithms that generalize to novel distributions~\cite{arjovsky2019invariant,lu2021invariant}, it suffices to just learn the causal relationship between the target variable and feature/latent variables while ignoring all other causal relationships.
Furthermore, in practice, there may be constraints on the interventions that one can perform and it is natural to prioritize the recovery of important causal relationships.
As such, in many practical situations, one is interested in learning the causal relationship only for a \emph{subset} of the edges of the causal graph while minimizing the number of interventions.

In this work, we formally initiate the study of this question by studying two fundamental problems under the assumptions of
ideal interventions\footnote{Ideal interventions assume hard interventions (forcefully setting a variable value) and the ability to obtain as many interventional samples as desired, ensuring that we always recover the directions of all edges cut by interventions. Without this assumption, we may fail to correctly infer some arc directions and our algorithms will only succeed with some success probability.},
faithfulness\footnote{Faithfulness assumes that independencies that occur in data does not occur due to ``cancellations'' in the functional relationships, but rather due to the causal graph structure. It is known \cite{meek1995strong,spirtes2000causation} that, under many natural parameterizations and settings, the set of unfaithful parameters for any given causal DAG has zero Lebesgue measure (i.e.\ faithfulness holds; see also \cite[Section 3.2]{zhang2002strong} for a discussion about faithfulness). However, one should be aware that the faithfulness assumption may be violated in reality \cite{andersen2013expect,uhler2013geometry}, especially in the presence of sampling errors in the finite sample regime.},
and causal sufficiency\footnote{Under causal sufficiency, there are no hidden confounders (i.e.\ unobserved common causes to the observed variables). While causal sufficiency may not always hold, it is still a reasonable one to make in certain applications such as studying gene regulatory networks (e.g.\ see \cite{wang2017permutation}).}: verification and search for learning the subset of edges in the causal graph, which we formally define below.

\begin{definition}[Subset verification problem]
Given a DAG $G = (V,E)$ and target edges $T \subseteq E$, find the minimum size/cost intervention set $\cI \subseteq 2^V$ such that $T \subseteq R(G,\cI)$.
\end{definition}

\begin{definition}[Subset search problem]
Given an \emph{unknown} ground truth DAG $G^* = (V,E)$ and target edges $T \subseteq E$, find the minimum size/cost intervention set $\cI \subseteq 2^V$ such that $T \subseteq R(G^*,\cI)$.
\end{definition}

\emph{Assuming $G$ was the ground truth}, we use $R(G,\cI)$ to denote the set of recovered arc directions due to interventions $\cI$ performed on an unoriented graph\footnote{The graph on which we perform interventions may not be completely oriented but a partially oriented one that is consistent with the Markov equivalence class of $G$. See \cref{sec:prelim} for a more accurate definition.}, and $\nu_1(G, T)$ to denote the minimum number of interventions needed to fully orient edges in $T$.
The above definitions are natural generalizations of the standard well-studied verification and search problems when $T = E$.
For $T=E$, the authors of \cite{choo2022verification} gave an efficient algorithm to compute the verification set of size $\nu_1(G, E)$ using the notion of covered edges and also provided an adaptive search algorithm that orients the whole causal graph using $\cO(\log n \cdot \nu_1(G, E))$ interventions.
Just as how the verification number $\nu_1(G,E)$ is a natural lower bound for the search problem, the \emph{subset} verification number $\nu_1(G,T)$ also serves as a lower bound for the \emph{subset} search problem.

\subsection{Our Contributions}

Despite being a simple generalization, the approach of \cite{choo2022verification} fails to directly extend to the subset verification and search problems. In our work, we show the following.

\subsubsection{Subset verification}

We give an efficient dynamic programming algorithm to compute a minimal subset verification set, and also extend to more general settings involving non-atomic interventions and additive vertex costs.

\subsubsection{Subset search}

We provide an explicit family of graphs $G$ and subsets $T$ such that the subset verification number $\nu_{1}(G,T)=1$ while any search algorithm needs $\mathrm{vc}(T)$ interventions to orient edges in $T$ against an adaptive adversary, where $\mathrm{vc}(T)$ is the size of the minimum vertex cover of $T$.
Thus, \emph{no} subset search algorithm has a better approximation ratio than $\mathrm{vc}(T)$ when compared with the $\nu_{1}(G,T)$ in general.
Furthermore, as $\nu_{1}(G,T) \leq \mathrm{vc}(T)$, it is trivial to design an algorithm to achieve this approximation ratio.
Meanwhile, in the special case where $T$ is a subset of all edges within a node-induced subgraph -- a setting that we believe is of practical interest -- we give a subset search algorithm that only incurs a logarithmic overhead in the \emph{size of the subgraph}, with respect to $\nu_1(G^*)$.
Note that here we compete against $\nu_1(G^*)$ and \emph{not} $\nu_1(G^*, T)$.

To obtain the above results, we provide a better understanding of how interventions work and prove several other interesting results.
We believe these results are fundamental and could be of independent interest.
For instance, we show that in the context of minimizing the number of ideal interventions used in causal graph discovery, it suffices to study DAGs without v-structures.
We also characterize the set of vertices orienting any given arc via Hasse diagrams.
We formalize these and other properties in \cref{sec:results}.

\subsection{Organization}

\cref{sec:prelim} contains preliminary notions and some related work.
We state our main results in \cref{sec:results}.
One of our results show that the subset verification problem is equivalent to a computational problem called \emph{interval stabbing on a rooted tree}, which we solve in \cref{sec:interval-stabbing-on-a-tree}.
Some experimental results are shown in \cref{sec:experiments}.
Full proofs and source code are given in the appendix.

\section{Preliminaries and related work}
\label{sec:prelim}

We write $A \,\dot\cup\, B$ to represent the disjoint union of two disjoint sets $A$ and $B$.

\subsection{Basic graph definitions}

Let $G = (V,E)$ be a graph on $|V| = n$ vertices.
We use $V(G)$, $E(G)$ and $A(G) \subseteq E(G)$ to denote its vertices, edges, and oriented arcs respectively.
The graph $G$ is said to be directed or fully oriented if $A(G) = E(G)$, and partially oriented otherwise.
For any two vertices $u,v \in V$, we write $u \sim v$ if these vertices are connected in the graph and $u \not\sim v$ otherwise.
To specify the arc directions, we use $u \to v$ or $u \gets v$.
For any subset $V' \subseteq V$ and $E' \subseteq E$, $G[V']$ and $G[E']$ denote the vertex-induced and edge-induced subgraphs respectively.

Consider an vertex $v \in V$ in a directed graph, let $\Pa(v), \Anc(v), \Des(v)$ denote the parents, ancestors and descendants of $v$ respectively.
Let vector $pa(v)$ denotes the values taken by $v$'s parents, e.g.\ if a parent node $v$ represents \textsc{Season}, then $pa(v) \in \{\textsc{Spring}, \textsc{Summer}, \textsc{Autumn}, \textsc{Winter}\}$.
Let $\Des[v] = \Des(v) \cup \{v\}$ and $\Anc[v] = \Anc(v) \cup \{v\}$.
We define $\Ch(v) \subseteq \Des(v)$ as the set of \emph{direct children} of $v$, that is, for any $w \in \Ch(v)$ there does \emph{not} exists $z \in V \setminus \{v,w\}$ such that $z \in \Des(v) \cap \Anc(w)$.
Note that, $\Ch(v) \subseteq \{w \in V : v \to w\} \subseteq \Des(v)$, i.e.\ $\Ch(v)$ is a subset of the standard notion of children in a directed graph, which in turn is a subset of all reachable vertices in a directed graph.

The \emph{skeleton} $skel(G)$ of a (partially oriented) graph $G$ is the underlying graph where all edges are made undirected.
A \emph{v-structure} refers to three distinct vertices $u,v,w \in V$ such that $u \to v \gets w$ and $u \not\sim w$.
A simple cycle is a sequence of $k \geq 3$ vertices where $v_1 \sim v_2 \sim \ldots \sim v_k \sim v_1$.
The cycle is partially directed if at least one of the edges is directed and all directed arcs are in the same direction along the cycle.
A partially directed graph is a \emph{chain graph} if it contains no partially directed cycle.
In the undirected graph $G[E \setminus A]$ obtained by removing all arcs from a chain graph $G$, each connected component in $G[E \setminus A]$ is called a \emph{chain component}.
We use $CC(G)$ to denote the set of chain components, where each $H \in CC(G)$ is a subgraph of $G$ and $V = \dot\cup_{H \in CC(G)} V(H)$.
For any partially directed graph, an \emph{acyclic completion / consistent extension} is an assignment of edge directions to all unoriented edges such that the resulting directed graph has no directed cycles.

Directed acyclic graphs (DAGs) are fully oriented chain graphs that are commonly used as graphical causal models \cite{pearl2009causality}, where vertices represents random variables and the joint probability density $f$ factorizes according to the Markov property:
$
f(v_1, \ldots, v_n) = \prod_{i=1}^n f(v_i \mid pa(v))
$.
We can associate a (not necessarily unique) \emph{valid permutation / topological ordering} $\pi : V \to [n]$ to any (partially directed) DAG such that oriented arcs $(u,v)$ satisfy $\pi(u) < \pi(v)$ and unoriented arcs $\{u,v\}$ can be oriented as $u \to v$ without forming directed cycles when $\pi(u) < \pi(v)$.

For any DAG $G$, we denote its \emph{Markov equivalence class} (MEC) by $[G]$ and \emph{essential graph} by $\cE(G)$.
DAGs in the same MEC $[G]$ have the same skeleton and essential graph $\cE(G)$ is a partially directed graph such that an arc $u \to v$ is directed if $u \to v$ in \emph{every} DAG in MEC $[G]$, and an edge $u \sim v$ is undirected if there exists two DAGs $G_1, G_2 \in [G]$ such that $u \to v$ in $G_1$ and $v \to u$ in $G_2$.
It is known that two graphs are Markov equivalent if and only if they have the same skeleton and v-structures \cite{verma1990,andersson1997characterization}.
An edge $u \sim v$ is a \emph{covered edge} \cite[Definition 2]{chickering2013transformational} if $\Pa(u) \setminus \{v\} = \Pa(v) \setminus \{u\}$.\footnote{On fully oriented graphs, the related notion of \emph{protected edges} \cite[Definition 3.2]{andersson1997characterization} is equivalent: an edge $a \sim b$ is \emph{not} protected if and only if it is a covered edge in $G[A]$.}

\subsection{Interventions and verifying sets}

An \emph{intervention} $S \subseteq V$ is an experiment where all variables $s \in S$ is forcefully set to some value, independent of the underlying causal structure.
An intervention is \emph{atomic} if $|S| = 1$ and \emph{bounded} if $|S| \leq k$ for some $k>0$; observational data is a special case where $S = \emptyset$.
The effect of interventions is formally captured by Pearl's do-calculus \cite{pearl2009causality}.
We call any $\cI \subseteq 2^V$ a \emph{intervention set}: an intervention set is a set of interventions where each intervention corresponds to a subset of variables.
An \emph{ideal intervention} on $S \subseteq V$ in $G$ induces an interventional graph $G_S$ where all incoming arcs to vertices $v \in S$ are removed \cite{eberhardt2012number}.
It is known that intervening on $S$ allows us to infer the edge orientation of any edge cut by $S$ and $V \setminus S$ \cite{eberhardt2007causation,hyttinen2013experiment,hu2014randomized,shanmugam2015learning,kocaoglu2017cost}.

For ideal interventions, an $\cI$-essential graph $\cE_{\cI}(G)$ of $G$ is the essential graph representing the Markov equivalence class of graphs whose interventional graphs for each intervention is Markov equivalent to $G_S$ for any intervention $S \in \cI$.
There are several known properties about $\cI$-essential graph properties \cite{hauser2012characterization,hauser2014two}:
Every $\cI$-essential graph is a chain graph with chordal\footnote{A chordal graph is a graph where every cycle of length at least 4 has a chord, which is an edge that is not part of the cycle but connects two vertices of the cycle; see \cite{blair1993introduction} for an introduction.} chain components.
This includes the case of $S = \emptyset$.
Orientations in one chain component do not affect orientations in other components.
In other words, to fully orient any essential graph $\cE(G^*)$, it is necessary and sufficient to orient every chain component in $\cE(G^*)$.

A \emph{verifying set} $\cI$ for a DAG $G \in [G^*]$ is an intervention set that fully orients $G$ from $\cE(G^*)$, possibly with repeated applications of Meek rules (see \cref{sec:appendix-meek-rules}).
In other words, for any graph $G = (V,E)$ and any verifying set $\cI$ of $G$, we have $\cE_{\cI}(G)[V'] = G[V']$ for \emph{any} subset of vertices $V' \subseteq V$.
Furthermore, if $\cI$ is a verifying set for $G$, then $\cI \cup S$ is also a verifying set for $G$ for any additional intervention $S \subseteq V$.
An \emph{subset verifying set} $\cI$ for a subset of \emph{target edges} $T \subseteq E$ in a DAG $G \in [G^*]$ is an intervention set that fully orients all arcs in $T$ given $\cE(G^*)$, possibly with repeated applications of Meek rules.
Note that the subset verifying set depends on the target edges \emph{and} the underlying ground truth DAG --- the subset verifying set for the same $T \subseteq E$ may differ across two different DAGs $G,G' \in [G^*]$ in the same Markov equivalence class.
While there may be multiple verifying sets in general, we are often interested in finding one with a minimum size or cost.

\begin{definition}[Minimum size/cost subset verifying set]
Let $w$ be a weight function on intervention sets.
An intervention set $\cI$ is called a subset verifying set for a subset of target edges $T \subseteq E$ in a DAG $G^*$ if all edges in $T$ are oriented in $\cE_{\cI}(G^*)$.
In the special case of $T = E$, we have $\cE_{\cI}(G^*) = G^*$.
$\cI$ is a \emph{minimum size (resp.\ cost) subset verifying set} if some edge in $T$ remains unoriented in $\cE_{\cI'}(G^*)$ for any $|\cI'| < |\cI|$ (resp.\ for any $w(\cI') < w(\cI)$).
\end{definition}

While restricting to interventions of size at most $k$, the \emph{minimum verification number} $\nu_k(G, T)$ denotes the size of the minimum size subset verifying set for any DAG $G \in [G^*]$ and subset of target edges $T \subseteq E$.
That is, any revealed arc directions when performing interventions on $\cE(G^*)$ respects $G$.
We write $\nu_1(G, T)$ when we restrict to atomic interventions.
When $k = 1$ and $T = E$, \cite{choo2022verification} tells us that it is necessary and sufficient to intervene on a minimum vertex cover of the covered edges in $G$.

For any intervention set $\cI \subseteq 2^V$, we write $R(G, \cI) = A(\cE_{\cI}(G)) \subseteq E$ to mean the set of oriented arcs in the $\cI$-essential graph of a DAG $G$.
Under this notation, we see that the directed arcs in the partially directed graph $\cE_{\cI}(G)$ can be expressed as $A(\cE_{\cI}(G)) = R(G,\cI)$.
For cleaner notation, we write $R(G,I)$ for single interventions $\cI = \{I\}$ for some $I \subseteq V$, and $R(G,v)$ for single atomic interventions $\cI = \{\{v\}\}$ for some $v \in V$.
The following lemma\footnote{While \cite{ghassami2018budgeted} studies atomic interventions, their proof extends to non-atomic intervention sets, and even the observational case where the intervention set could be $\emptyset$. However, there are some minor fixable bugs in their proof. For completeness, we provide a shorter fixed proof of \cref{lem:recovered-union} in \cref{sec:appendix-stronger-GSKB}.} implies that the combined knowledge of two intervention sets do not further trigger any Meek rules.

\begin{restatable}[Modified lemma 2 of \cite{ghassami2018budgeted}]{lemma}{recoveredunion}
\label{lem:recovered-union}
For any DAG $G = (V,E)$ and any two intervention sets $\cI_1, \cI_2 \subseteq 2^V$, we have $R(G, \cI_1 \cup \cI_2) = R(G, \cI_1) \cup R(G, \cI_2)$.
\end{restatable}

We define $R^{-1}_1(G, a \to b) \subseteq V$ and $R^{-1}_k(G, a \to b) \subseteq 2^V$ to refer to interventions orienting an arc $a \to b$:
\begin{align*}
R^{-1}_1(G, a \to b) &= \{v \in V: a \to b \in R(G,v)\}\\
R^{-1}_k(G, a \to b) &= \{I \subseteq V: |I| \leq k, a \to b \in R(G,I)\}
\end{align*}
For any oriented arc $a \to b \in A$, we let $R^{-1}_1(a \to b) = V$ and $R^{-1}_k(a \to b) = \{I \subseteq V: |I| \leq k\}$.
For any subset $S \subseteq E$, we denote $R(G,S) \subseteq E$ as the set of oriented arcs in the essential graph of $G$ if we orient $S$, along with the v-structure arcs in $G$, then apply Meek rules till convergence.
In particular, when $S = \{(u,v) : u \in \cI \text{ or } v \in \cI\} \subseteq E$ is the set of incident edges to some vertex set $\cI \subseteq V$, then $R(G,S) = R(G,\cI)$ are precisely the oriented arcs in the interventional essential graph $\cE_{\cI}(G)$.
Furthermore, if $S$ is a \emph{superset} of the set of incident edges to some vertex set $\cI \subseteq V$, then $R(G,\cI) \subseteq R(G,S)$.
When we use the $R(G, \cdot)$ notation, we will be explicit about its type -- whether $\cdot$ is a subset of vertices $V$, a subset of a subset of vertices $2^V$, or a subset of edges $E$.

\subsection{Hasse diagrams}

\begin{definition}[Directed Hasse diagram]
Any poset $(\cX, \leq)$ can be \emph{uniquely} represented by a \emph{directed Hasse diagram} $H_{(X,\leq)}$, a directed graph where each element in $\cX$ is a vertex and
there is an arc $y \to x$ whenever $y$ covers $x$ for any two elements $x,y \in \cX$.
We call these arcs as \emph{Hasse arcs}.
\end{definition}

Any DAG $G = (V,E)$ induces a poset on the vertices $V$ with respect to the ancestral relationships in the graph: $x \leq_{\Anc} y$ whenever $x \in \Anc[y]$.
Since ``covers'' correspond to ``direct children'' for DAGs, we will say ``$y$ is a direct child of $x$'' instead of ``$x$ covers $y$'' to avoid confusion with the notion of covered edges.
We will use $H_G = H_{(V, \leq_{\Anc})}$ to denote the Hasse diagram corresponding to a DAG $G$.
The Hasse diagram $H_G$ can be computed in polynomial time \cite{aho1972transitive} and may have multiple roots (vertices without incoming arcs) in general.
Background on posets and related notions are given in \cref{sec:appendix-hasse}.

\subsection{Related work}

As discussed in \cref{sec:introduction}, the most relevant prior work to ours is \cite{choo2022verification} where they studied the problems of verification and search under the \emph{special case} of $T = E$.

Other related works using non-atomic interventions include:
\cite{hu2014randomized} showed that $G^*$ can be identified using $\cO(\log(\log(n)))$ unbounded randomized interventions in expectation;
\cite{shanmugam2015learning} showed that $\cO(\frac{n}{k} \log(\log(k)))$ bounded sized interventions suffices.

Other related works in the setting of additive vertex costs include:
\cite{ghassami2018budgeted} studied the problem of maximizing the number of oriented edges given a budget of atomic interventions;
\cite{kocaoglu2017cost,lindgren2018experimental} studied the problem of finding a minimum cost (bounded size) intervention set to identify $G^*$;
\cite{lindgren2018experimental} showed that computing the minimum cost intervention set is NP-hard and gave search algorithms with constant approximation factors.

\section{Results}
\label{sec:results}

Here, we present our main results for the subset verification and subset search problems: we provide an efficient algorithm to compute a minimal subset verifying set for any given subset of target edges and show asymptotically matching upper and lower bounds for subset search.

In \cref{sec:sufficient-to-study-without-v-structures}, we show that it suffices to study the subset search and verification problems on DAGs without v-structures.
Then, in \cref{sec:hasse}, we consider DAGs without v-structures and show several interesting properties regarding their Hasse diagram $H_G$.
In \cref{sec:atomic}, we use structural properties on $H_G$ to show that the subset verification problem is equivalent to another problem called \emph{interval stabbing on a rooted tree}, which we solve in \cref{sec:interval-stabbing-on-a-tree}.
By further exploiting the structural properties of $H_G$, we show how to extend our subset verification results to the settings of non-atomic interventions and additive vertex costs in \cref{sec:nonatomic}.
Finally, we present our results for the subset search problem in \cref{sec:search}.

We believe that the properties presented in \cref{sec:sufficient-to-study-without-v-structures} and \cref{sec:hasse} are of independent interest and have applications beyond the subset verification and search problems.

\subsection{Sufficient to study DAGs without v-structures}
\label{sec:sufficient-to-study-without-v-structures}

Here, we state some structural properties of interventional essential graphs $\cE_{\cI}(G)$. These properties enable us to ignore v-structures and justify the study of the subset verification and search problems solely on DAGs without v-structures. Recall the observational essential graph $\cE(G)$ is an interventional essential graph for $\cI = {\emptyset}$.

\begin{definition}[Oriented subgraphs and recovered parents]
For any interventional set $\cI \subseteq 2^V$ and $u \in V$, define $G^{\cI} = G[E \setminus R(G,\cI)]$ as the \emph{fully directed} subgraph DAG induced by the \emph{unoriented arcs} in $G$ and $\Pa_{G,\cI}(u) = \{x \in V: x \to u \in R(G,\cI)\}$ as the recovered parents of $u$ by $\cI$.
\end{definition}

\begin{restatable}[Properties of interventional essential graphs]{theorem}{interventionalessentialgraphproperties}
\label{thm:properties}
Consider a DAG $G = (V,E)$ and intervention sets $\cA, \cB \subseteq 2^V$.
Then, the following statements are true:
\begin{enumerate}
    \item $skel(G^{\cA})$ is exactly the chain components of $\cE_{\cA}(G)$.
    \item $G^{\cA}$ does not have new v-structures.\footnote{
    While classic results \cite{andersson1997characterization,hauser2012characterization} tell us that chain components of interventional essential graphs are chordal, i.e.\ $\cE(G)[E \setminus A]$ is a chordal graph, it is not immediately obvious why such edge-induced subgraphs cannot have v-structures in any of the DAGs compatible with $\cE(G)$.
    Here, we formalize this fact.
    }
    \item For any two vertices $u$ and $v$ in the same chain component of $\cE_{\cA}(G)$, we have $\Pa_{G,\cA}(u) = \Pa_{G,\cA}(v)$.
    \item If the arc $u \to v \in R(G,\cA)$, then $u$ and $v$ belong to different chain components of $\cE_{\cA}(G)$.
    \item Any acyclic completion of $\cE(G^{\cA})$ that does not form new v-structures can be combined with $R(G,\cA)$ to obtain a valid DAG that belongs to both $\cE(G)$ and $\cE_{\cA}(G)$.\footnote{Stated in a different language in \cite[Proposition 16]{hauser2012characterization}.}
    \item $R(G^{\cA},\cB) = R(G,\cB) \setminus R(G,\cA)$.
    \item $R(G,\cA \cup \cB) = R(G^{\cA},\cB) \;\dot\cup\; R(G,\cA)$.
    \item \raggedright $R(G,\cA \cup \cB) = R(G^{\cA},\cB) \;\dot\cup\ R(G^{\cB},\cA) \;\dot\cup\; (R(G,\cA) \cap R(G,\cB))$.
    \item $R(G,\emptyset)$ does not contain any covered edge of $G$.
\end{enumerate}
\end{restatable}

From prior work \cref{lem:recovered-union}, we have $R(G, \cA \cup \cB) = R(G, \cA) \cup R(G, \cB)$ for any two interventions $\cA$ and $\cB$.
Informally, this means that combining prior \emph{orientations} will not trigger Meek rules.
Meanwhile, \cref{thm:properties} says that the \emph{adjacencies} will also not, thus we can simplify the causal graphs by removing any oriented edges before performing further interventions.
\cref{fig:interventional-essential-graph} gives an illustration.

An important implication of \cref{thm:properties} for verification and search problems is that it suffices to solve these problems only on DAGs without v-structures.
As any oriented arcs in the observational graph can be removed \emph{before performing any interventions}, the optimality of the solution is unaffected since $R(G,\cI) = R(G^{\emptyset},\cI) \;\dot\cup\; R(G, \emptyset)$, where $G^{\emptyset}$ is the graph obtained after removing all the oriented arcs in the observational essential graph due to v-structures.

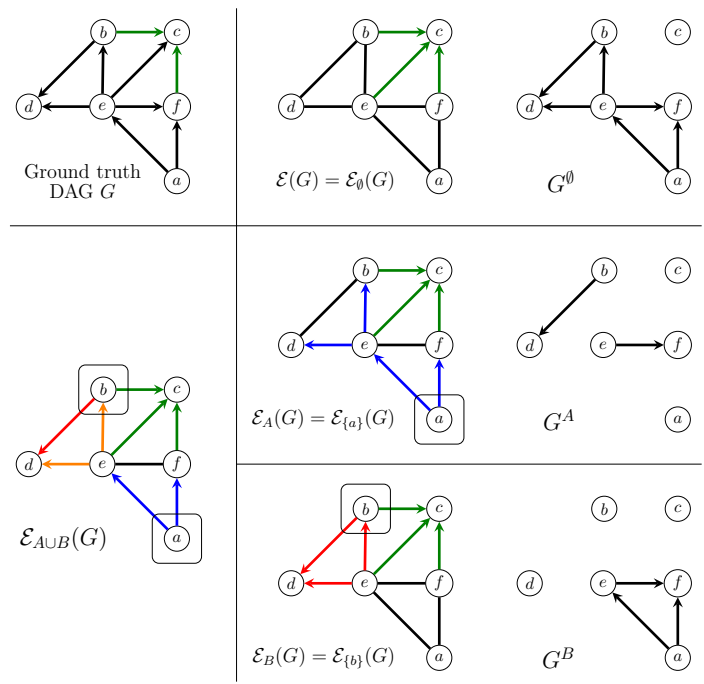
\begin{figure}[h]
\centering
\resizebox{0.55\linewidth}{!}{%
\begin{tikzpicture}
%
%
\node[draw, circle, minimum size=15pt, inner sep=2pt] at (-0.5,0) (Gc) {\normalsize $c$};
\node[draw, circle, minimum size=15pt, inner sep=2pt, left=of Gc] (Gb) {\normalsize $b$};
\node[draw, circle, minimum size=15pt, inner sep=2pt, below=of Gc] (Gf) {\normalsize $f$};
\node[draw, circle, minimum size=15pt, inner sep=2pt, left=of Gf] (Ge) {\normalsize $e$};
\node[draw, circle, minimum size=15pt, inner sep=2pt, left=of Ge] (Gd) {\normalsize $d$};
\node[draw, circle, minimum size=15pt, inner sep=2pt, below=of Gf] (Ga) {\normalsize $a$};
\draw[ultra thick, -stealth] (Ga) -- (Ge);
\draw[ultra thick, -stealth] (Ga) -- (Gf);
\draw[ultra thick, green!50!black, -stealth] (Gb) -- (Gc);
\draw[ultra thick, -stealth] (Gb) -- (Gd);
\draw[ultra thick, -stealth] (Ge) -- (Gb);
\draw[ultra thick, -stealth] (Ge) -- (Gc);
\draw[ultra thick, -stealth] (Ge) -- (Gd);
\draw[ultra thick, -stealth] (Ge) -- (Gf);
\draw[green!50!black, ultra thick, -stealth] (Gf) -- (Gc);
\node[align=center, left=10pt of Ga] {\large Ground truth\\ \large DAG $G$};

%
%
\node[draw, circle, minimum size=15pt, inner sep=2pt] at (5,0) (EGc) {\normalsize $c$};
\node[draw, circle, minimum size=15pt, inner sep=2pt, left=of EGc] (EGb) {\normalsize $b$};
\node[draw, circle, minimum size=15pt, inner sep=2pt, below=of EGc] (EGf) {\normalsize $f$};
\node[draw, circle, minimum size=15pt, inner sep=2pt, left=of EGf] (EGe) {\normalsize $e$};
\node[draw, circle, minimum size=15pt, inner sep=2pt, left=of EGe] (EGd) {\normalsize $d$};
\node[draw, circle, minimum size=15pt, inner sep=2pt, below=of EGf] (EGa) {\normalsize $a$};
\draw[ultra thick] (EGa) -- (EGe);
\draw[ultra thick] (EGa) -- (EGf);
\draw[green!50!black, ultra thick, -stealth] (EGb) -- (EGc);
\draw[ultra thick] (EGb) -- (EGd);
\draw[ultra thick] (EGe) -- (EGb);
\draw[green!50!black, ultra thick, -stealth] (EGe) -- (EGc);
\draw[ultra thick] (EGe) -- (EGd);
\draw[ultra thick] (EGe) -- (EGf);
\draw[green!50!black, ultra thick, -stealth] (EGf) -- (EGc);
\node[text centered, left=15pt of EGa] {\large $\cE(G) = \cE_{\emptyset}(G)$};

%
%
\node[draw, circle, minimum size=15pt, inner sep=2pt] at (10,0) (G-empty-c) {\normalsize $c$};
\node[draw, circle, minimum size=15pt, inner sep=2pt, left=of G-empty-c] (G-empty-b) {\normalsize $b$};
\node[draw, circle, minimum size=15pt, inner sep=2pt, below=of G-empty-c] (G-empty-f) {\normalsize $f$};
\node[draw, circle, minimum size=15pt, inner sep=2pt, left=of G-empty-f] (G-empty-e) {\normalsize $e$};
\node[draw, circle, minimum size=15pt, inner sep=2pt, left=of G-empty-e] (G-empty-d) {\normalsize $d$};
\node[draw, circle, minimum size=15pt, inner sep=2pt, below=of G-empty-f] (G-empty-a) {\normalsize $a$};
\draw[ultra thick, -stealth] (G-empty-a) -- (G-empty-e);
\draw[ultra thick, -stealth] (G-empty-a) -- (G-empty-f);
\draw[ultra thick, -stealth] (G-empty-b) -- (G-empty-d);
\draw[ultra thick, -stealth] (G-empty-e) -- (G-empty-b);
\draw[ultra thick, -stealth] (G-empty-e) -- (G-empty-d);
\draw[ultra thick, -stealth] (G-empty-e) -- (G-empty-f);
\node[text centered, left=50pt of G-empty-a] {\Large $G^{\emptyset}$};

%
%
\node[draw, circle, minimum size=15pt, inner sep=2pt] at (-0.5,-7.5) (EABGc) {\normalsize $c$};
\node[draw, circle, minimum size=15pt, inner sep=2pt, left=of EABGc] (EABGb) {\normalsize $b$};
\node[draw, circle, minimum size=15pt, inner sep=2pt, below=of EABGc] (EABGf) {\normalsize $f$};
\node[draw, circle, minimum size=15pt, inner sep=2pt, left=of EABGf] (EABGe) {\normalsize $e$};
\node[draw, circle, minimum size=15pt, inner sep=2pt, left=of EABGe] (EABGd) {\normalsize $d$};
\node[draw, circle, minimum size=15pt, inner sep=2pt, below=of EABGf] (EABGa) {\normalsize $a$};
\draw[blue, ultra thick, -stealth] (EABGa) -- (EABGe);
\draw[blue, ultra thick, -stealth] (EABGa) -- (EABGf);
\draw[green!50!black, ultra thick, -stealth] (EABGb) -- (EABGc);
\draw[red, ultra thick, -stealth] (EABGb) -- (EABGd);
\draw[orange, ultra thick, -stealth] (EABGe) -- (EABGb);
\draw[green!50!black, ultra thick, -stealth] (EABGe) -- (EABGc);
\draw[orange, ultra thick, -stealth] (EABGe) -- (EABGd);
\draw[ultra thick] (EABGe) -- (EABGf);
\draw[green!50!black, ultra thick, -stealth] (EABGf) -- (EABGc);
\node[draw, fit=(EABGa), rounded corners, inner sep=7pt] {};
\node[draw, fit=(EABGb), rounded corners, inner sep=7pt] {};
\node[text centered, left=30pt of EABGa] {\Large $\cE_{A \cup B}(G)$};

%
%
\node[draw, circle, minimum size=15pt, inner sep=2pt] at (5,-5) (EAGc) {\normalsize $c$};
\node[draw, circle, minimum size=15pt, inner sep=2pt, left=of EAGc] (EAGb) {\normalsize $b$};
\node[draw, circle, minimum size=15pt, inner sep=2pt, below=of EAGc] (EAGf) {\normalsize $f$};
\node[draw, circle, minimum size=15pt, inner sep=2pt, left=of EAGf] (EAGe) {\normalsize $e$};
\node[draw, circle, minimum size=15pt, inner sep=2pt, left=of EAGe] (EAGd) {\normalsize $d$};
\node[draw, circle, minimum size=15pt, inner sep=2pt, below=of EAGf] (EAGa) {\normalsize $a$};
\draw[blue, ultra thick, -stealth] (EAGa) -- (EAGe);
\draw[blue, ultra thick, -stealth] (EAGa) -- (EAGf);
\draw[green!50!black, ultra thick, -stealth] (EAGb) -- (EAGc);
\draw[ultra thick] (EAGb) -- (EAGd);
\draw[blue, ultra thick, -stealth] (EAGe) -- (EAGb);
\draw[green!50!black, ultra thick, -stealth] (EAGe) -- (EAGc);
\draw[blue, ultra thick, -stealth] (EAGe) -- (EAGd);
\draw[ultra thick] (EAGe) -- (EAGf);
\draw[green!50!black, ultra thick, -stealth] (EAGf) -- (EAGc);
\node[draw, fit=(EAGa), rounded corners, inner sep=7pt] {};
\node[text centered, left=15pt of EAGa] {\large $\cE_{A}(G) = \cE_{\{a\}}(G)$};

%
%
\node[draw, circle, minimum size=15pt, inner sep=2pt] at (10,-5) (G-A-c) {\normalsize $c$};
\node[draw, circle, minimum size=15pt, inner sep=2pt, left=of G-A-c] (G-A-b) {\normalsize $b$};
\node[draw, circle, minimum size=15pt, inner sep=2pt, below=of G-A-c] (G-A-f) {\normalsize $f$};
\node[draw, circle, minimum size=15pt, inner sep=2pt, left=of G-A-f] (G-A-e) {\normalsize $e$};
\node[draw, circle, minimum size=15pt, inner sep=2pt, left=of G-A-e] (G-A-d) {\normalsize $d$};
\node[draw, circle, minimum size=15pt, inner sep=2pt, below=of G-A-f] (G-A-a) {\normalsize $a$};
\draw[ultra thick, -stealth] (G-A-b) -- (G-A-d);
\draw[ultra thick, -stealth] (G-A-e) -- (G-A-f);
\node[text centered, left=50pt of G-A-a] {\Large $G^{A}$};

%
%
\node[draw, circle, minimum size=15pt, inner sep=2pt] at (5,-10) (EBGc) {\normalsize $c$};
\node[draw, circle, minimum size=15pt, inner sep=2pt, left=of EBGc] (EBGb) {\normalsize $b$};
\node[draw, circle, minimum size=15pt, inner sep=2pt, below=of EBGc] (EBGf) {\normalsize $f$};
\node[draw, circle, minimum size=15pt, inner sep=2pt, left=of EBGf] (EBGe) {\normalsize $e$};
\node[draw, circle, minimum size=15pt, inner sep=2pt, left=of EBGe] (EBGd) {\normalsize $d$};
\node[draw, circle, minimum size=15pt, inner sep=2pt, below=of EBGf] (EBGa) {\normalsize $a$};
\draw[ultra thick] (EBGa) -- (EBGe);
\draw[ultra thick] (EBGa) -- (EBGf);
\draw[green!50!black, ultra thick, -stealth] (EBGb) -- (EBGc);
\draw[red, ultra thick, -stealth] (EBGb) -- (EBGd);
\draw[red, ultra thick, -stealth] (EBGe) -- (EBGb);
\draw[green!50!black, ultra thick, -stealth] (EBGe) -- (EBGc);
\draw[red, ultra thick, -stealth] (EBGe) -- (EBGd);
\draw[ultra thick] (EBGe) -- (EBGf);
\draw[green!50!black, ultra thick, -stealth] (EBGf) -- (EBGc);
\node[draw, fit=(EBGb), rounded corners, inner sep=7pt] {};
\node[text centered, left=15pt of EBGa] {\large $\cE_{B}(G) = \cE_{\{b\}}(G)$};

%
%
\node[draw, circle, minimum size=15pt, inner sep=2pt] at (10,-10) (G-B-c) {\normalsize $c$};
\node[draw, circle, minimum size=15pt, inner sep=2pt, left=of G-B-c] (G-B-b) {\normalsize $b$};
\node[draw, circle, minimum size=15pt, inner sep=2pt, below=of G-B-c] (G-B-f) {\normalsize $f$};
\node[draw, circle, minimum size=15pt, inner sep=2pt, left=of G-B-f] (G-B-e) {\normalsize $e$};
\node[draw, circle, minimum size=15pt, inner sep=2pt, left=of G-B-e] (G-B-d) {\normalsize $d$};
\node[draw, circle, minimum size=15pt, inner sep=2pt, below=of G-B-f] (G-B-a) {\normalsize $a$};
\draw[ultra thick, -stealth] (G-B-a) -- (G-B-e);
\draw[ultra thick, -stealth] (G-B-a) -- (G-B-f);
\draw[ultra thick, -stealth] (G-B-e) -- (G-B-f);
\node[text centered, left=50pt of G-B-a] {\Large $G^{B}$};

%
%
\draw ($(EGa)!0.5!(EAGc) + (-9,0)$) -- ($(EGa)!0.5!(EAGc) + (5.5,0)$);
\draw ($(EAGa)!0.5!(EBGc) + (-4.25,0)$) -- ($(EAGa)!0.5!(EBGc) + (5.5,0)$);
\draw ($(EGc) + (-4.25,0.5)$) -- ($(EBGa) + (-4.25,-0.5)$);
\end{tikzpicture}
}
\caption{
Example for \cref{thm:properties}.
Here, recovered edges $R(G, \cdot)$ are colored while the black edges are the hidden arc directions.
Since $b \to c \gets f$ is a v-structure in $G$, these edges are oriented in the observational essential graph $\cE(G)$ and so Meek rule R3 orients $e \to c$ in $\cE(G)$.
Intervening on $A = \{a\}$ orients $\{a \to e, a \to f\}$ and Meek rule R1 further orients $\{e \to b, e \to d\}$.
Intervening on $B = \{b\}$ orients $\{e \to b, b \to d\}$ and Meek rule R2 further orients $\{e \to d\}$.
Observe that $R(G^A, B) = \{b \to d\}$, $R(G^B, A) = \{a \to e, a \to f\}$, and $R(G,A) \cap R(G,B) \setminus R(G, \emptyset) = \{e \to b, e \to d\}$.
}
\label{fig:interventional-essential-graph}
\end{figure}

\subsection{Hasse diagrams of DAGs without v-structures}
\label{sec:hasse}

Here, we show some interesting properties of the Hasse diagrams for DAGs without v-structures.
\begin{restatable}{lemma}{vstructfreeiffHasseisatree}
\label{lem:v-struct-free-iff-Hasse-is-a-tree}
A DAG $G = (V,E)$ is a single connected component without v-structures if and only if the Hasse diagram $H_G$ is a directed tree with a unique root vertex.
\end{restatable}
As it is known \cite[Lemma 23]{hauser2012characterization} that any DAG without v-structures whose skeleton is a connected chordal graph has exactly one source vertex, \cref{lem:v-struct-free-iff-Hasse-is-a-tree} is not entirely surprising.
However, it enables us to properly define the notion of rooted subtrees in a Hasse diagram.

\begin{definition}[Rooted subtree]
\label{def:rooted-subtree}
Let $H_G$ be a Hasse diagram of a single component DAG $G = (V,E)$ without v-structures.
By \cref{lem:v-struct-free-iff-Hasse-is-a-tree}, $H_G$ is a rooted tree.
For any vertex $y \in V$, the rooted subtree $T_y$ has vertices $V(T_y) = \{u \in V: y \in \Anc[u]\}$ and edges $E(T_y) = \{a \to b : a,b \in V(T_y)\}$.
See \cref{fig:rooted-subtree} for an illustration.
\end{definition}

\begin{figure}[h]
\centering
\resizebox{0.3\linewidth}{!}{%
\begin{tikzpicture}
\node[draw, circle, minimum size=15pt, inner sep=2pt] at (0,0) (root) {\small $r$};
\node[draw, circle, minimum size=15pt, inner sep=2pt] at ($(root) + (0,-2)$) (paw) {$z$};
\node[draw, circle, minimum size=15pt, inner sep=2pt] at ($(paw) + (0,-1)$) (w) {\small $w$};
\node[draw, circle, minimum size=15pt, inner sep=2pt] at ($(w) + (-2,-1)$) (w-child1) {\small $a$};
\node[draw, circle, minimum size=15pt, inner sep=2pt] at ($(w) + (2,-1)$) (w-child2) {\small $b$};
\node[draw, circle, minimum size=15pt, inner sep=2pt] at ($(w) + (0,-1)$) (y) {\small $y$};
\node[draw, circle, minimum size=15pt, inner sep=2pt] at ($(y) + (0,-1.25)$) (c) {\small $c$};

\node[draw, shape border uses incircle, isosceles triangle, shape border rotate=90] at ($(root) + (-2,-1)$) (root-subtree1) {};
\node[draw, shape border uses incircle, isosceles triangle, shape border rotate=90] at ($(root) + (-1,-1)$) (root-subtree2) {};
\node[draw, shape border uses incircle, isosceles triangle, shape border rotate=90] at ($(root) + (1,-1)$) (root-subtree3) {};
\node[draw, shape border uses incircle, isosceles triangle, shape border rotate=90] at ($(root) + (2,-1)$) (root-subtree4) {};
\node[draw, shape border uses incircle, isosceles triangle, shape border rotate=90] at ($(w-child1) + (0,-1.25)$) (w-subtree1) {};
\node[draw, shape border uses incircle, isosceles triangle, shape border rotate=90] at ($(w-child2) + (0,-1.25)$) (w-subtree2) {};
\node[draw, shape border uses incircle, isosceles triangle, shape border rotate=90] at ($(y) + (-1,-1.25)$) (y-subtree1) {};
\node[draw, shape border uses incircle, isosceles triangle, shape border rotate=90] at ($(y) + (1,-1.25)$) (y-subtree2) {};
\node[draw, shape border uses incircle, isosceles triangle, shape border rotate=90] at ($(c) + (0,-1)$) (c-subtree) {};

\draw[thick, -stealth] (root) -- node[pos=0.4, fill=white, text=black] {$\vdots$} (paw);
\draw[thick, -stealth] (paw) -- (w);
\draw[thick, -stealth] (w) -- (y);
\draw[thick, -stealth] (y) -- (c);
\draw[dashed, thick, -stealth] (w) to[out=315,in=45] (c);
\draw[thick, -stealth] (root) -- (root-subtree1.north);
\draw[thick, -stealth] (root) -- (root-subtree2.north);
\draw[thick, -stealth] (root) -- (root-subtree3.north);
\draw[thick, -stealth] (root) -- (root-subtree4.north);
\draw[thick, -stealth] (w) to[out=180,in=90] (w-child1);
\draw[thick, -stealth] (w) to[out=0,in=90] (w-child2);
\draw[thick, -stealth] (w-child1) -- (w-subtree1.north);
\draw[thick, -stealth] (w-child2) -- (w-subtree2.north);
\draw[thick, -stealth] (y) -- (y-subtree1.north);
\draw[thick, -stealth] (y) -- (y-subtree2.north);
\draw[thick, -stealth] (c) -- (c-subtree.north);

\node[text centered, right=5pt of paw] {\small $z = \Pa(w)$};
\node[red, draw, fit=(root)(paw), rounded corners, label={[red,yshift=-25pt]left:{\small $\Anc(w)$}}] {};
\node[blue, draw, fit=(y)(w-child1)(w-child2), rounded corners, inner sep=6pt, label={[blue,xshift=60pt]left:{\small $\Ch(w)$}}] {};
\node[orange, draw, fit=(y)(c)(y-subtree1)(y-subtree2)(c-subtree), rounded corners, inner sep=10pt, label={[orange,xshift=30pt,yshift=20pt]below:{\small $T_y$}}] {};
\end{tikzpicture}
}
\caption{
A Hasse diagram $H_G$ of some DAG $G$ with root $r$ where triangles represent unexpanded subtrees.
For a vertex $w$, $\Anc(w)$ is the set of vertices along the unique path from $r$ to $w$ and $z = \Pa(w)$ is the vertex directly before $w$.
The direct children of $w$ are $\Ch(w) = \{a,b,y\}$.
If the arc $w \to c$ exists in $G$, it will \emph{not} appear in $H_G$ because $w \to y \to c$ exists, i.e.\ $c \not\in \Ch(w)$.
The rooted subtree $T_y$ at $y$ includes \emph{all} the nodes that have $y$ as an ancestor.
}
\label{fig:rooted-subtree}
\end{figure}
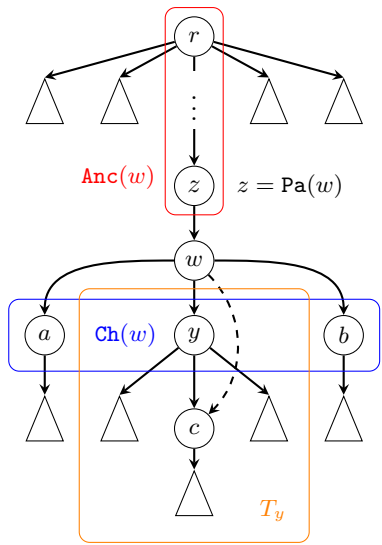

Using rooted subtrees, we prove several structural properties regarding the arc directions that are recovered by an atomic intervention, cumulating into \cref{thm:orienting-vertices-form-an-interval} which states that the set $R^{-1}_1(G, u \to v)$ of vertices whose intervention recovers $u \to v$ forms a consecutive sequence of vertices in some branch in the Hasse diagram $H_G$.

\begin{restatable}{theorem}{orientingverticesformaninterval}
\label{thm:orienting-vertices-form-an-interval}
Let $G = (V,E)$ be a DAG without v-structures and $u \to v$ be an unoriented arc in $\cE(G)$.
Then, $R^{-1}_1(G, u \to v) = \Des[w] \cap \Anc[v]$ for some $w \in \Anc[u]$.
\end{restatable}

By \cref{thm:orienting-vertices-form-an-interval}, we only need to intervene on some vertex within each sequence of Hasse arcs.
Meanwhile, \cref{lem:covered-edges-are-Hasse-edges} tells us that covered edges correspond directly to an interval involving only the endpoints.
Thus, we see that our subset verification algorithm is a non-trivial generalization of \cite{choo2022verification}.

\begin{restatable}{lemma}{coverededgesareHasseedges}
\label{lem:covered-edges-are-Hasse-edges}
If $G$ be a DAG without v-structures, then the covered edges of $G$ are a subset of the Hasse edges in $H_G$.
\end{restatable}

\subsection{Subset verification on DAGs without v-structures with atomic interventions}
\label{sec:atomic}

In this section, we show that the atomic subset verification problem is equivalent to the problem of interval stabbing on a rooted tree that we define next.
For a DAG $G$ without v-structures, let $H_G$ be its rooted Hasse tree.

For any rooted tree $\widehat{G} = (V,E)$, an ordered pair $[u,v]_{\widehat{G}} \in V \times V$ is called an \emph{interval} if $u \in \Anc(v)$.
If the graph is clear from context, we will drop the subscript $\widehat{G}$.
We say that a vertex $z \in V$ \emph{stabs} an interval $[u,v]$ if and only if $z \in \Des[u] \cap \Anc[v]$, and that a subset $S \subseteq V$ stabs $[a,b]$ if $S$ has a vertex that stabs it.

Interpreting \cref{thm:orienting-vertices-form-an-interval} with respect to the definition of an interval, we see that every edge $u \to v$ can be associated with some interval $[w,v]_{H_G}$, for some $w \in \Anc[u]$, such that $u \to v \in R(G, \cI)$ if and only if $\cI$ stabs $[w,v]_{H_G}$.
As such, we can reduce the subset verification problem on DAGs without v-structures to the following problem.

\begin{definition}[Interval stabbing problem on a rooted tree]
\label{defn:IC}
Given a rooted tree $\widehat{G} = (V,E)$ with root $r \in V$ and a set $\cJ$ of intervals of the form $[u,v]$, find a set $\cI \subseteq V$ of minimum size such that $[u,v] \cap \cI \neq \emptyset$ for all $[u,v] \in \cJ$.
\end{definition}

The interval stabbing problem on a rooted tree can be viewed both as a special case of the set cover problem, and as a generalization of the interval stabbing problem on a line.
The former is NP-hard \cite{karp1972reducibility},  while the latter can be solved using a polynomial time greedy algorithm (e.g.\ see \cite[Chapter 4, Exercise 4]{erickson2019algorithms}).
The next result shows that one can reduce the subset verification problem on DAGs without v-structures to the interval stabbing problem.

\begin{restatable}{lemma}{subsetverificationashardasintervalstabbing}
\label{lem:subset-verification-as-hard-as-interval-stabbing}
Let $G = (V,E)$ be a connected DAG without v-structures, $H$ be the Hasse tree of $G$, and $T \subseteq E$ be a subset of target edges.
Then, there exists a set of intervals $\cJ \subseteq 2^{V \times V}$ such that any solution to minimum interval stabbing problem on $(H,\cJ)$ is a solution to the minimum sized atomic subset verification set $(G,T)$.
\end{restatable}

\cref{sec:sufficient-to-study-without-v-structures} tells us that we can ignore arc orientations due to v-structures, thus we can apply \cref{thm:orienting-vertices-form-an-interval} and \cref{lem:subset-verification-as-hard-as-interval-stabbing} to reduce the problem to an instance of interval stabbing on a rooted tree.
As \cref{lem:exists-poly-time-algo-for-interval-stabbing} tells us that this can be solved efficiently, we obtain an efficient algorithm for the subset verification problem (\cref{thm:exists-poly-time-algo-for-subset-verification}).

\begin{restatable}{theorem}{existspolytimealgoforintervalstabbing}
\label{lem:exists-poly-time-algo-for-interval-stabbing}
There exists a polynomial time algorithm for solving the interval stabbing problem on a rooted tree.
\end{restatable}

\begin{restatable}{theorem}{existspolytimealgoforsubsetverification}
\label{thm:exists-poly-time-algo-for-subset-verification}
For any DAG $G = (V,E)$ and subset of target edges $T \subseteq E$, there exists a polynomial time algorithm to compute the minimal sized atomic subset verifying set.
\end{restatable}

Interestingly, \emph{any} instance of interval stabbing on a rooted tree can also be reduced in polynomial time to an instance of subset verification on DAGs without v-structures.

\begin{restatable}{lemma}{intervalstabbingashardassubsetverification}
\label{lem:interval-stabbing-as-hard-as-subset-verification}
Let $H$ be a rooted tree and $\cJ \subseteq 2^{V \times V}$ be a set of intervals.
Then, there exists a connected DAG $G = (V,E)$ without v-structures and a subset $T \subseteq E$ of edges such that any solution to the minimum sized atomic subset verification set $(G,T)$ is a solution to minimum interval stabbing problem on $(H,\cJ)$.
\end{restatable}

\subsection{Subset verification on DAGs without v-structures with bounded size interventions and additive costs}
\label{sec:nonatomic}

Here we extend our results to the setting of bounded size interventions, where each intervention involves at most $k$ vertices, and additive vertex costs, where each vertex has an associative cost $w(v)$ for intervening.
Formally, one can define a weight function on the vertices $w: V \to \R$ which overloads to $w(S) = \sum_{v \in S} w(v)$ on interventions and $w(\cI) = \sum_{S \in \cI} S$ on intervention sets.
To trade off minimum cost and minimum size, we study how to minimize the following objective function that has been studied by \cite{kocaoglu2017cost,ghassami2018budgeted,choo2022verification}:
\begin{equation}
\label{eq:generalized-cost}
\alpha \cdot w(\cI) + \beta \cdot |\cI| \qquad \text{where $\alpha, \beta \geq 0$}
\end{equation}

To extend the verification results to the bounded size and additive node costs settings, \cite{choo2022verification} exploited the fact that the covered edges were a forest, and thus bipartite, to construct non-atomic interventions sets by grouping vertices from the minimum size atomic verification set.
In our problem setting, the target edges $T \subseteq E$ may not even involve covered edges of $G$ and it is a priori unclear how one can hope to apply the above-mentioned strategy of \cite{choo2022verification}.

Motivated by the fact that $R(G,\cI) = R(G,S)$ for any atomic verifying set $\cI \subseteq V$ and set of incident arcs $S = \{(u,v) : u \in \cI \text{ or } v \in \cI\} \subseteq E$, we show the following:

\begin{restatable}{lemma}{thereisatreesubset}
\label{lem:there-is-a-tree-subset}
Let $G = (V,E)$ be a DAG without v-structures and $S \subseteq E$.
Then, there exists a subset $S' \subseteq E$ computable in polynomial time such that $G[S']$ is a forest, $R(G,S) \subseteq R(G,S')$, and $\bigcup_{(u,v) \in S'} \{u,v\} \subseteq \bigcup_{(u,v) \in S} \{u,v\}$.
\end{restatable}

By invoking \cref{lem:there-is-a-tree-subset} with $S$ as the incident arcs of the minimum size subset verification set $\cI$, we can obtain a 2-coloring of $\cI$ with respect to $S'$.
Thus, we can apply the ``greedy grouping'' generalization strategy of \cite{choo2022verification} to achieve the similar guarantees as them, generalizing their results beyond $T = E$.

\begin{restatable}{theorem}{boundedsizeextension}
Fix an essential graph $\cE(G^*)$ and $G \in [G^*]$.
If $\nu_1(G,T) = \ell$, then $\nu_k(G,T) \geq \lceil \frac{\ell}{k} \rceil$ and there exists a polynomial time algorithm to compute a bounded size intervention set $\cI$ of size $|\cI| \leq \lceil \frac{\ell}{k} \rceil + 1$.
\end{restatable}

\begin{restatable}{theorem}{additivecostextension}
Fix an essential graph $\cE(G^*)$ and $G \in [G^*]$.
Suppose the optimal bounded size intervention set that minimizes \cref{eq:generalized-cost} costs $OPT$.
Then, there exists a polynomial time algorithm that computes a bounded size intervention set with total cost $OPT + 2 \beta$.
\end{restatable}

\subsection{Subset search on DAGs without v-structures}
\label{sec:search}

While a vertex cover of the target edges   is a trivial upper bound for atomic subset search, we show that one needs to perform that many number of atomic interventions asymptotically in the worst case when facing an adaptive adversary.
That is, the adversary gets to see the interventions made by the adaptive algorithm and gets to choose the ground truth DAG among the set of all DAGs that are consistent with the already revealed information.

\begin{restatable}{lemma}{vertexcoveristrivialupperbound}
\label{lem:vertex-cover-is-trivial-upper-bound}
Given a subset of target edges $T \subseteq E$, intervening on the vertices in a vertex cover of $T$ one-by-one will fully orient all edges in $T$.
\end{restatable}

\begin{restatable}{lemma}{vertexcoverisworstcasenecessary}
\label{lem:vertexcoverisworstcasenecessary}
Fix any integer $n \geq 1$.
There exists a fully unoriented essential graph on $2n$ vertices a subset $T \subseteq E$ on $n$ edges such that the size of the minimum vertex cover of $T$ is $\textrm{vc}(T)$ and any algorithm needs at least $\textrm{vc}(T)-1$ number atomic interventions to orient all the edges in $T$ against an adaptive adversary that reveals arc directions consistent with a DAG $G^* \in [G]$ with $\nu_1(G^*, T) = 1$.
\end{restatable}

\cref{fig:adaptive-lower-bound} in \cref{sec:appendix-subsetsearch} illustrates our construction for \cref{lem:vertexcoverisworstcasenecessary}, where $\textrm{vc}(T) \in \omega(n)$.

On the other hand, if we restrict the class of target edges to be edges within a node-induced subgraph $H$, then we can actually obtain the following non-trivial search result.

\begin{definition}[Relevant nodes]
Fix a DAG $G^* = (V,E)$ and arbitrary subset $V' \subseteq V$.
For any intervention set $\cI \subseteq V$ and resulting interventional essential graph $\cE_{\cI}(G^*)$, we define the \emph{relevant nodes} $\rho(\cI, V') \subseteq V'$ as the set of nodes within $V'$ that is adjacent to some unoriented arc within the node-induced subgraph $\cE_{\cI}(G^*)[V']$.
\end{definition}

\begin{restatable}{theorem}{subsetsearchonnodeinducedsubgraph}
\label{thm:subset-search-on-node-induced-subgraph}
Fix an interventional essential graph $\cE_{\cI}(G^*)$ of an unknown underlying DAG $G^*$ and let $H$ be an arbitrary node-induced subgraph.
There exists an algorithm that runs in polynomial time and computes an atomic intervention set $\cI'$ in a deterministic and adaptive manner such that $\cE_{\cI \cup \cI'}(G^*)[V(H)] = G^*[V(H)]$ and $|\cI'| \in \cO(\log(|\rho(\cI, V(H))|) \cdot \nu_1(G^*, E))$.
\end{restatable}

Note that \cref{thm:subset-search-on-node-induced-subgraph} compares against $\nu_1(G^*, E)$ and not $\nu_1(G^*, E(H))$ and that the observational essential graph simply corresponds to the special case where $\cI = \emptyset$.
Since node-induced subgraphs of a chordal graph are also chordal, the chain components in $\cE_{\cI}(G^*)[V(H)]$ are chordal.
Our algorithm \texttt{SubsetSearch}, given in \cref{sec:appendix-subsetsearch}, is a generalization of \cite[Algorithm 1]{choo2022verification}, where we employ the \emph{weighted} chordal graph separator guarantees from \cite{gilbert1984separatorchordal}.
Just like \cite{choo2022verification}, \texttt{SubsetSearch} can be also be generalized to perform bounded size interventions on the computed clique separators via the labelling scheme of \cite[Lemma 1]{shanmugam2015learning}.

\begin{restatable}{theorem}{subsetsearchonnodeinducedsubgraphbounded}
\label{thm:subset-search-on-node-induced-subgraph-bounded}
Fix an interventional essential graph $\cE_{\cI}(G^*)$ of an unknown underlying DAG $G^*$ and let $H$ be an arbitrary node-induced subgraph.
There exists an algorithm that runs in polynomial time and computes a bounded size intervention set $\cI'$, where each intervention involves at most $k \geq 1$ nodes, in a deterministic and adaptive manner such that $\cE_{\cI \cup \cI'}(G^*)[V(H)] = G^*[V(H)]$ and $|\cI'| \in \cO(\log(|\rho(\cI, V(H))|) \cdot \log (k) \cdot \nu_k(G^*, E))$.
\end{restatable}

\section{Interval stabbing problem on a rooted tree}
\label{sec:interval-stabbing-on-a-tree}

Here, we formulate a recurrence relation for the interval stabbing problem on a rooted tree and give an efficient dynamic programming implementation in \cref{sec:appendix-dp}.

To formally describe the recurrence relation, we will use the following definitions to partition the given set of intervals.
Given a set of intervals $\cJ$, we define the following sets with respect to an arbitrary vertex $v \in V$:
\begin{align*}
    E_v &= \{[a,b] \in \cJ : b=v\} && \text{(End with $v$)}\\
    M_v &= \{[a,b] \in \cJ : v \in (a,b)\} && \text{(Middle with $v$)}\\
    S_v &= \{[a,b] \in \cJ : a=v\} && \text{(Start with $v$)}\\
    W_v &= \{[a,b] \in \cJ : a,b \in V(T_{v}) \setminus \{v\}\} && \text{(Without $v$)}\\
    I_v &= E_v \cup M_v \cup S_v \cup W_v && \text{(Intersect $T_v$)}\\
    B_v &= S_v \cup W_v && \text{(Back of $I_v$)}\\
    C_v &= E_v \cup M_v \cup S_v && \text{(Covered by $v$)}
\end{align*}
Note that $I_v$ includes all the intervals in $\cJ$ that intersect with the subtree $T_v$ (i.e.\ has some vertex in $V(T_v)$) and $C_v$ includes all the intervals that will be covered whenever $v \in \cI$.
Observe that $I_y \subseteq I_v$ for any $y \in \Des(v)$.
See \cref{fig:line-example} for an example illustrating these definitions.

\begin{figure}[htbp]
\centering
\resizebox{0.7\linewidth}{!}{%
\begin{tikzpicture}
\node[draw, circle, minimum size=15pt, inner sep=2pt] at (0,0) (v1) {\small $v_1$};
\node[draw, circle, minimum size=15pt, inner sep=2pt, right=of v1] (v2) {\small $v_2$};
\node[draw, circle, minimum size=15pt, inner sep=2pt, right=of v2] (v3) {\small $v_3$};
\node[draw, circle, minimum size=15pt, inner sep=2pt, right=of v3] (v4) {\small $v_4$};
\node[draw, circle, minimum size=15pt, inner sep=2pt, right=of v4] (v5) {\small $v_5$};
\node[draw, circle, minimum size=15pt, inner sep=2pt, right=of v5] (v6) {\small $v_6$};
\node[draw, circle, minimum size=15pt, inner sep=2pt, right=of v6] (v7) {\small $v_7$};
\node[draw, circle, minimum size=15pt, inner sep=2pt, right=of v7] (v8) {\small $v_8$};
\draw[ultra thick, -stealth] (v1) -- (v2);
\draw[ultra thick, -stealth] (v2) -- (v3);
\draw[ultra thick, -stealth] (v3) -- (v4);
\draw[ultra thick, -stealth] (v4) -- (v5);
\draw[ultra thick, -stealth] (v5) -- (v6);
\draw[ultra thick, -stealth] (v6) -- (v7);
\draw[ultra thick, -stealth] (v7) -- (v8);

\node[draw, rounded corners, fit=(v4), inner sep=6pt] {};
\draw[thick, dashed, Bracket-Bracket] ($(v1) + (0,-1.75)$) -- node[midway, fill=white, text=black] {$J_1$} ($(v6) + (0,-1.75)$);
\draw[thick, dashed, Bracket-Bracket] ($(v2) + (0,-0.75)$) -- node[midway, fill=white, text=black] {$J_2$} ($(v4) + (0,-0.75)$);
\draw[thick, dashed, Bracket-Bracket] ($(v2) + (0,-1.25)$) -- node[midway, fill=white, text=black] {$J_3$} ($(v5) + (0,-1.25)$);
\draw[thick, dashed, Bracket-Bracket] ($(v4) + (0,-0.75)$) -- node[midway, fill=white, text=black] {$J_4$} ($(v7) + (0,-0.75)$);
\draw[thick, dashed, Bracket-Bracket] ($(v7) + (0,-0.75)$) -- node[midway, fill=white, text=black] {$J_5$} ($(v8) + (0,-0.75)$);
\end{tikzpicture}
}
\caption{
Consider the rooted tree $\widehat{G}$ with $v_1 \to \ldots \to v_8$ and $\cJ = \{J_1, \ldots, J_5\}$, where $J_1 = [v_1,v_6]$, $J_2 = [v_2,v_4]$, $J_3 = [v_2,v_5]$, $J_4 = [v_4,v_7]$, and $J_5 = [v_7,v_8]$.
Then, $E_{v_4} = \{J_2\}$, $M_{v_4} = \{J_1, J_3\}$, $S_{v_4} = \{J_4\}$, $W_{v_4} = \{J_5\}$.
}
\label{fig:line-example}
\end{figure}
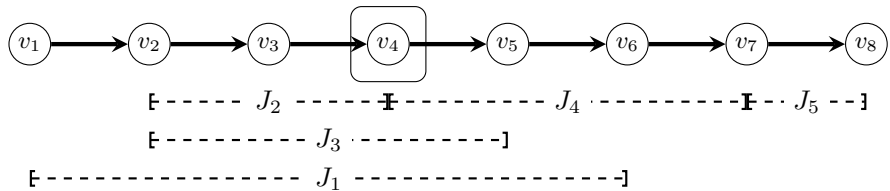

To solve the interval stabbing problem, we perform recursion from the root towards the leaves, solving subproblems defined on subsets of the intervals that are still ``relevant'' at each subtree.
More formally, for any set of intervals $U \subseteq \cJ$, let $\opt(U, v)$ denote the \emph{size} of the optimum solution to stab all the intervals in $U$ using only vertices $V(T_v)$ in the subtree $T_v$ rooted at $v$.
There are three possible cases while recursing from the root towards the leaves:

\begin{description}
    \item[Case 1.] If $U \cap E_v \neq \emptyset$, then $v$ \emph{must} be in any valid solution output and we recurse on the set $(U \setminus C_v) \cap I_y$ for subtree $T_y$ rooted at each child $y \in \Ch(v)$.
    \item[Case 2.] If $U \cap E_v = \emptyset$ and $v$ is in the output, then we \emph{can} recurse on the set $(U \setminus C_v) \cap I_y$ for subtree $T_y$ rooted at each child $y \in \Ch(v)$.
    \item[Case 3.] If $U \cap E_v = \emptyset$ and $v$ is \emph{not} in the output, then we \emph{need to} recurse on the set $U \cap I_y$ subtree $T_y$ rooted at each child $y \in \Ch(v)$.
\end{description}

For any $v \in V$ and $y \in \Ch(v)$, we have $C_v \cap I_y \subseteq E_y \cup M_y$ by definition.
So, $(U \setminus C_v) \cap I_y = U \cap B_y$.
The correctness of the first case is trivial while \cref{lem:recurrence-correctness} formalizes the correctness of the second and third cases.

\begin{restatable}{lemma}{recurrencecorrectness}
\label{lem:recurrence-correctness}
At least one of the following must hold for any optimal solution $\cI$ to the interval stabbing problem with respect to ordering $\pi$ and any vertex $v \in V$ with $E_v = \emptyset$:
\begin{enumerate}
    \item Either $v \in \cI$ or $\cI$ includes some ancestor of $v$.
    \item For $y \in \Ch(v)$ such that $C_v \cap I_y \neq \emptyset$, we must have $w_{v,y} \in \cI$ for some $w_{v,y} \in \Des(v) \cap \Anc[b_{v,y}]$, where $[a_{v,y},b_{v,y}] = \argmin_{[a,b] \in U \cap C_v \cap I_y} \{ \pi(b) \}$.
\end{enumerate}
\end{restatable}

Therefore, we have the following recurrence relation:
\begin{equation}
\label{eq:dp-recurrence}
\opt(U, v) =
\begin{cases}
\infty & \text{if $U \not\subseteq I_v$}\\
\alpha_v & \text{if $U \subseteq I_v$, $U \cap E_v \neq \emptyset$}\\
\min\{\alpha_v, \beta_v\} & \text{if $U \subseteq I_v$, $U \cap E_v = \emptyset$}
\end{cases}
\end{equation}
\begin{align*}
\text{where}\qquad
\alpha_v &= 1 + \sum_{y \in \Ch(v)} \opt(U \cap B_y, y)\\
\beta_v &= \sum_{y \in \Ch(v)} \opt(U \cap I_y, y)
\end{align*}

That is, we must pick $v \in \cI$ whenever $E_v \neq \emptyset$, while $\alpha_v$ and $\beta_v$ correspond to the decisions of picking $v$ into the output and ignoring $v$ from the output respectively.
Then, $\opt(\cJ, r)$ is the optimum solution size to the interval stabbing problem, where $r$ as the root of the given rooted tree.

In \cref{sec:appendix-dp}, we explain how to implement \cref{eq:dp-recurrence} efficiently using dynamic programming.
To do so, we first compute the Euler tour data structure on $G$ and use it to define an ordering on $\cJ$ so that our state space ranges over the indices of a sorted array instead of a subset of intervals.

\section{Experiments and implementation}
\label{sec:experiments}

Here, we discuss some experiments conducted on synthetic graphs.
For full details and source code, see \cref{sec:appendix-experiments}.

\subsection{Experiment 1: \emph{Randomly chosen} target edges}

We implemented our subset verification algorithm by invoking our dynamic programming algorithm for the interval stabbing problem, given in \cref{sec:appendix-dp}.
Using experiments on synthetic random graphs, we empirically show that subset verification numbers $\nu_1(G,T)$ decreases from the full verification number $\nu_1(G,E)$ as the $T$ decreases (see \cref{fig:exp-p03}), as expected.
Despite the trend suggested in \cref{fig:exp-p03}, the number of target edges is typically \emph{not} a good indication for the number of interventions needed to be performed and one can always construct examples where $|T'| > |T|$ but $\nu(G,T') \not> \nu(G,T)$.
For example, for a subset $T \subseteq E$, we have $\nu(G^*, T') = \nu(G^*, T)$ if $T' \supset T$ is obtained by adding edges that are already oriented by orienting $T$.
Instead, the number of ``independent target edges''\footnote{Akin to ``linearly independent vectors'' in linear algebra.} is a more appropriate measure.

\begin{figure}[h]
\centering
\includegraphics[width=0.7\linewidth]{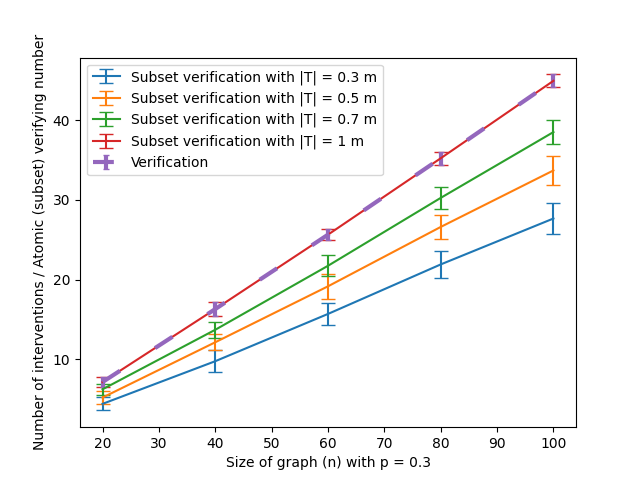}
\caption{
For each graph $G^*$ with $|E| = m$ edges, we sampled a random subset $T \subseteq E$ of various sizes.
The subset verification numbers $\nu_1(G^*, T)$ increases towards the full verification number $\nu_1(G^*, E)$ as the $|T|$ increases, and directly coincides with it in the special case of $|T| = m$.
}
\label{fig:exp-p03}
\end{figure}

\subsection{Experiment 2: Local causal graph discovery}

As motivated in \cref{sec:introduction}, many practical applications are interested in only recovering \emph{localized} causal relations for a fixed target variable of interest.
Unfortunately, existing full graph search algorithms are not tailored to recover directions of a given subset of edges.
In fact, one can create simple instances where the optimum number of interventions needed to perform the recovery task is just one, while a full graph search algorithm performs $\Omega(n)$ interventions\footnote{Suppose we wish to orient a single edge. Clearly single intervention on one of the endpoints suffice. Meanwhile, in the event that the target edge is \emph{not} in any 1/2-clique separator, the algorithm of \cite{choo2022verification} already incurs $\omega(G^*)$ interventions in the very first round, where $\omega(G^*)$ can be made arbitrarily large.}.

In \cref{fig:r-hop1}, we show the number of interventions needed to orient all edges within a $r$-hop neighborhood of some randomly chosen target node $v$.
We see that that node-induced subgraph search \texttt{SubsetSearch} uses less interventions than existing state-of-the-art full graph search algorithms, even when we terminate them as soon as all edges in $T$ are oriented.
We also give results for $r = 3$ in \cref{sec:appendix-experiments}.

\begin{figure}[h]
\centering
\includegraphics[width=0.7\linewidth]{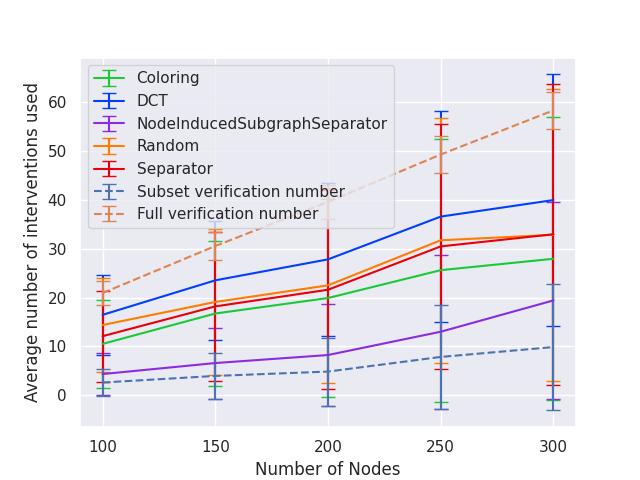}
\caption{\texttt{SubsetSearch} consistently uses less interventions than existing state-of-the-art full graph search algorithms when we only wish to orient edges within a 1-hop neighborhood of a randomly chosen target node $v$.}
\label{fig:r-hop1}
\end{figure}

\section{Conclusion and discussion}
\label{sec:conclusion}

Correctly identify causal relationships is a fundamental task both for understanding a system and for downstream tasks such as designing fair algorithms.
In many practical situations, the causal graph may be large and only a subset of causal relationships are important.
In this work, we give efficient algorithms for solving the subset verification and subset search problems under the standard causal inference assumptions (see \cref{sec:introduction}), generalizing the results of \cite{choo2022verification}.
However, if our assumptions are violated by the data, then wrong causal conclusions may be drawn and possibly lead to unintended downstream consequences.
Hence, it is of great interest to remove/weaken these assumptions while maintaining strong theoretical guarantees.

For search on a subset of target edges $T \subseteq E$, we showed that $\Omega(\textrm{vc}(T) \cdot \nu_1(G^*, T))$ interventions are necessary in general while $\cO(\log n \cdot \nu_1(G^*, E))$ interventions suffice when $T = E$.
This suggests that the verification number $\nu(G^*, \cdot)$ is perhaps too pessimistic of a benchmark to compare against in general, and we should instead compare against the ``best'' algorithm that does \emph{not} know $G^*$.

\subsubsection*{Acknowledgements}
This research/project is supported by the National Research Foundation, Singapore under its AI Singapore Programme (AISG Award No: AISG-PhD/2021-08-013).
KS was supported by a Stanford Data Science Scholarship, a Dantzig-Lieberman Research Fellowship and a Simons-Berkeley Research Fellowship.
Part of this work was done while the authors were visiting the Simons Institute for the Theory of Computing.
We would like to thank the AISTATS reviewers, Arnab Bhattacharyya, Themis Gouleakis, and Marcel Wien\"{o}bst for their valuable feedback, discussion, and writing suggestions.

\bibliographystyle{alpha}
\bibliography{refs}

\newpage
\appendix

\section{Meek rules}
\label{sec:appendix-meek-rules}

Meek rules are a set of 4 edge orientation rules that are sound and complete with respect to any given set of arcs that has a consistent DAG extension \cite{meek1995}.
Given any edge orientation information, one can always repeatedly apply Meek rules till a fixed point to maximize the number of oriented arcs.

\begin{definition}[Consistent extension]
A set of arcs is said to have a \emph{consistent DAG extension} $\pi$ for a graph $G$ if there exists a permutation on the vertices such that (i) every edge $\{u,v\}$ in $G$ is oriented $u \to v$ whenever $\pi(u) < \pi(v)$, (ii) there is no directed cycle, (iii) all the given arcs are present.
\end{definition}

\begin{definition}[The four Meek rules \cite{meek1995}, see \ref{fig:meek-rules} for an illustration]
\hspace{0pt}
\begin{description}
    \item [R1] Edge $\{a,b\} \in E \setminus A$ is oriented as $a \to b$ if $\exists$ $c \in V$ such that $c \to a$ and $c \not\sim b$.
    \item [R2] Edge $\{a,b\} \in E \setminus A$ is oriented as $a \to b$ if $\exists$ $c \in V$ such that $a \to c \to b$.
    \item [R3] Edge $\{a,b\} \in E \setminus A$ is oriented as $a \to b$ if $\exists$ $c,d \in V$ such that $d \sim a \sim c$, $d \to b \gets c$, and $c \not\sim d$.
    \item [R4] Edge $\{a,b\} \in E \setminus A$ is oriented as $a \to b$ if $\exists$ $c,d \in V$ such that $d \sim a \sim c$, $d \to c \to b$, and $b \not\sim d$.
\end{description}
\end{definition}

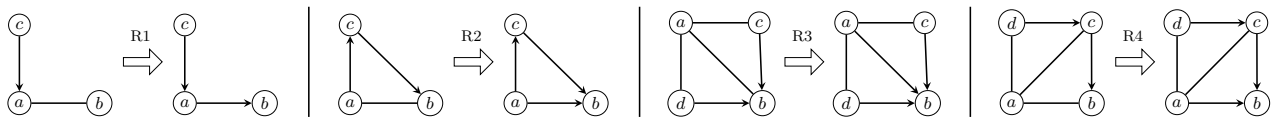
\begin{figure}[htbp]
\centering
\resizebox{\linewidth}{!}{%
\begin{tikzpicture}
%
%
\node[draw, circle, inner sep=2pt] at (0,0) (R1a-before) {\small $a$};
\node[draw, circle, inner sep=2pt, right=of R1a-before] (R1b-before) {\small $b$};
\node[draw, circle, inner sep=2pt, above=of R1a-before](R1c-before) {\small $c$};
\draw[thick, -stealth] (R1c-before) -- (R1a-before);
\draw[thick] (R1a-before) -- (R1b-before);

\node[draw, circle, inner sep=2pt] at (3,0) (R1a-after) {\small $a$};
\node[draw, circle, inner sep=2pt, right=of R1a-after] (R1b-after) {\small $b$};
\node[draw, circle, inner sep=2pt, above=of R1a-after](R1c-after) {\small $c$};
\draw[thick, -stealth] (R1c-after) -- (R1a-after);
\draw[thick, -stealth] (R1a-after) -- (R1b-after);

\node[single arrow, draw, minimum height=2em, single arrow head extend=1ex, inner sep=2pt] at (2.2,0.75) (R1arrow) {};
\node[above=5pt of R1arrow] {\footnotesize R1};

%
%
\node[draw, circle, inner sep=2pt] at (6,0) (R2a-before) {\small $a$};
\node[draw, circle, inner sep=2pt, right=of R2a-before] (R2b-before) {\small $b$};
\node[draw, circle, inner sep=2pt, above=of R2a-before](R2c-before) {\small $c$};
\draw[thick, -stealth] (R2a-before) -- (R2c-before);
\draw[thick, -stealth] (R2c-before) -- (R2b-before);
\draw[thick] (R2a-before) -- (R2b-before);

\node[draw, circle, inner sep=2pt] at (9,0) (R2a-after) {\small $a$};
\node[draw, circle, inner sep=2pt, right=of R2a-after] (R2b-after) {\small $b$};
\node[draw, circle, inner sep=2pt, above=of R2a-after](R2c-after) {\small $c$};
\draw[thick, -stealth] (R2a-after) -- (R2c-after);
\draw[thick, -stealth] (R2c-after) -- (R2b-after);
\draw[thick, -stealth] (R2a-after) -- (R2b-after);

\node[single arrow, draw, minimum height=2em, single arrow head extend=1ex, inner sep=2pt] at (8.2,0.75) (R2arrow) {};
\node[above=5pt of R2arrow] {\footnotesize R2};

%
%
\node[draw, circle, inner sep=2pt] at (12,0) (R3d-before) {\small $d$};
\node[draw, circle, inner sep=2pt, above=of R3d-before](R3a-before) {\small $a$};
\node[draw, circle, inner sep=2pt, right=of R3a-before] (R3c-before) {\small $c$};
\node[draw, circle, inner sep=2pt, right=of R3d-before](R3b-before) {\small $b$};
\draw[thick, -stealth] (R3c-before) -- (R3b-before);
\draw[thick, -stealth] (R3d-before) -- (R3b-before);
\draw[thick] (R3c-before) -- (R3a-before) -- (R3d-before);
\draw[thick] (R3a-before) -- (R3b-before);

\node[draw, circle, inner sep=2pt] at (15,0) (R3d-after) {\small $d$};
\node[draw, circle, inner sep=2pt, above=of R3d-after](R3a-after) {\small $a$};
\node[draw, circle, inner sep=2pt, right=of R3a-after] (R3c-after) {\small $c$};
\node[draw, circle, inner sep=2pt, right=of R3d-after](R3b-after) {\small $b$};
\draw[thick, -stealth] (R3c-after) -- (R3b-after);
\draw[thick, -stealth] (R3d-after) -- (R3b-after);
\draw[thick] (R3c-after) -- (R3a-after) -- (R3d-after);
\draw[thick, -stealth] (R3a-after) -- (R3b-after);

\node[single arrow, draw, minimum height=2em, single arrow head extend=1ex, inner sep=2pt] at (14.2,0.75) (R3arrow) {};
\node[above=5pt of R3arrow] {\footnotesize R3};

%
%
\node[draw, circle, inner sep=2pt] at (18,0) (R4a-before) {\small $a$};
\node[draw, circle, inner sep=2pt, above=of R4a-before](R4d-before) {\small $d$};
\node[draw, circle, inner sep=2pt, right=of R4d-before] (R4c-before) {\small $c$};
\node[draw, circle, inner sep=2pt, right=of R4a-before](R4b-before) {\small $b$};
\draw[thick, -stealth] (R4d-before) -- (R4c-before);
\draw[thick, -stealth] (R4c-before) -- (R4b-before);
\draw[thick] (R4d-before) -- (R4a-before) -- (R4c-before);
\draw[thick] (R4a-before) -- (R4b-before);

\node[draw, circle, inner sep=2pt] at (21,0) (R4a-after) {\small $a$};
\node[draw, circle, inner sep=2pt, above=of R4a-after](R4d-after) {\small $d$};
\node[draw, circle, inner sep=2pt, right=of R4d-after] (R4c-after) {\small $c$};
\node[draw, circle, inner sep=2pt, right=of R4a-after](R4b-after) {\small $b$};
\draw[thick, -stealth] (R4d-after) -- (R4c-after);
\draw[thick, -stealth] (R4c-after) -- (R4b-after);
\draw[thick] (R4d-after) -- (R4a-after) -- (R4c-after);
\draw[thick, -stealth] (R4a-after) -- (R4b-after);

\node[single arrow, draw, minimum height=2em, single arrow head extend=1ex, inner sep=2pt] at (20.2,0.75) (R4arrow) {};
\node[above=5pt of R4arrow] {\footnotesize R4};

\draw[thick] (5.25,1.75) -- (5.25,-0.25);
\draw[thick] (11.25,1.75) -- (11.25,-0.25);
\draw[thick] (17.25,1.75) -- (17.25,-0.25);
\end{tikzpicture}
}
\caption{An illustration of the four Meek rules}
\label{fig:meek-rules}
\end{figure}

There exists an algorithm \cite[Algorithm 2]{pmlr-v161-wienobst21a} that runs in $\cO(d \cdot |E|)$ time and computes the closure under Meek rules, where $d$ is the degeneracy of the graph skeleton\footnote{A $d$-degenerate graph is an undirected graph in which every subgraph has a vertex of degree at most $d$. Note that the degeneracy of a graph is typically smaller than the maximum degree of the graph.}.

The following results tell us that Meek rules can only ``propagate downstream''.

\begin{lemma}
\label{lem:if-v-orients-then-v-to-b-path-exists}
Let $G = (V,E)$ be a DAG.
If $v \in R^{-1}(G, a \to b)$, then there exists a directed path from $v$ to $b$ in $G$.
That is, $v \in \Anc[b]$.
\end{lemma}
\begin{proof}
Since $v \in R^{-1}(G, a \to b)$, there must be at least one new arc in the Meek rule (see \ref{fig:meek-rules}) that fired to orient $a \to b$ due to $v$.
Let us perform induction on the number of hops from $v$.

Base case ($v$ appears in all R1 to R4):
\begin{description}
    \item[R1] $v$ can only be $c$ and we have $c \to a \to b$
    \item[R2] $v$ can only be $c$ and we have $c \to b$
    \item[R3] $v$ can either be $c$ or $d$. In either case, we have $c \to b$, $d \to b$
    \item[R4] $v$ can either be $c$ or $d$. In either case, we have $d \to c \to b$.
\end{description}

Inductive case:
\begin{description}
    \item[R1] We must have $v \in R^{-1}(G, c \to a)$. By induction, there is a path from $v$ to $a$, so there is a path from $v$ to $b$.
    \item[R2] We must have $v \in R^{-1}(G, a \to c)$ or $v \in R^{-1}(G, c \to b)$. By induction, there is a path from $v$ to $c$ or to $b$. In either case, there is a path from $v$ to $b$.
    \item[R3] We must have $v \in R^{-1}(G, c \to b)$ or $v \in R^{-1}(G, d \to b)$. In either case, there is a path from $v$ to $b$ by induction.
    \item[R4] We must have $v \in R^{-1}(G, d \to c)$ or $v \in R^{-1}(G, c \to b)$. By induction, there is a path from $v$ to $c$ or to $b$. In either case, there is a path from $v$ to $b$.
\end{description}
\end{proof}

We can also show a arc version of \cref{lem:if-v-orients-then-v-to-b-path-exists}.

\begin{lemma}
\label{lem:downstream-arc}
Let $G = (V,E)$ be a DAG.
If an arc $u \to v$ is used to orient $a \to b$, then $b \in \Des[v]$.
\end{lemma}
\begin{proof}
Suppose $u \to v$ appears in the Meek rule that orients $a \to b$.
Observe that $v \leq_{\Anc} b$ in all cases.
\end{proof}
\section{Hasse diagrams and transitive reductions}
\label{sec:appendix-hasse}

\begin{definition}[Partial order]
The tuple $(\cX, \leq)$ is a partially ordered set (a.k.a.\ poset) whenever the partial order $\leq$ on a set $\cX$ satisfies three properties:
(1) Reflexivity: For all $x \in \cX$, $x \leq x$;
(2) Anti-symmetric: For all $x,y \in \cX$, if $x \leq y$ and $y \leq x$, then $x = y$;
(3) Transitivity: For all $x,y,z \in \cX$, if $x \leq y$ and $y \leq z$, then $x \leq z$.
Note that there may be pairs of elements in $X$ that are incomparable.
For any two elements $x,y \in \cX$, we say that $y$ \emph{covers} $x$ if $x \leq y$ and there is no $z \in \cX \setminus \{x,y\}$ such that $x \leq z \leq y$.
\end{definition}

\begin{definition}[Directed Hasse diagram]
Any poset $(\cX, \leq)$ can be \emph{uniquely} represented by a \emph{directed Hasse diagram} $H_{(X,\leq)}$, a directed graph where each element in $\cX$ is a vertex and
there is an arc $y \to x$ whenever $y$ covers $x$ for any two elements $x,y \in \cX$.
We call these arcs as \emph{Hasse arcs}.
\end{definition}

\begin{definition}[Transitive reduction]
A \emph{transitive reduction} of a directed graph $G = (V, E)$ is another directed graph $G^t = (V, E')$ with minimum sized $|E'|$ such that there is a directed path from $u$ to $v$ in $G$ if and only if there is a directed path from $u$ to $v$ in $G^t$ for any $u,v \in V$.
\end{definition}

Any DAG $G = (V,E)$ induces a poset on the vertices $V$ with respect to the ancestral relationships in the graph: $x \leq_{\Anc} y$ whenever $x \in \Anc[y]$.
Furthermore, it is known (e.g.\ see \cite{aho1972transitive}) that the transitive reduction $G^t$ of a DAG $G$ is \emph{unique}, is defined on a subset of edges (i.e.\ $E' \subseteq E$), is polynomial time computable, and is exactly the Hasse diagram $H_{(V, \leq_{\Anc})}$ defined with respect to $(V, \leq_{\Anc})$.
Since ``covers'' correspond to ``direct children'' for DAGs, we will say ``$y$ is a direct child of $x$'' instead of ``$x$ covers $y$'' to avoid confusion with the notion of covered edges.
In the rest of the paper, we will use $H_G = H_{(V, \leq_{\Anc})}$ to denote the Hasse diagram corresponding to a DAG $G = (V,E)$.
A vertex without incoming arcs in a Hasse diagram is called a \emph{root}.
In general, there may be multiple roots.
\section{Proof of \cref{lem:recovered-union}}
\label{sec:appendix-stronger-GSKB}

While \cite{ghassami2018budgeted} studies atomic interventions, their proof extends to non-atomic intervention sets, and even the observational case where the intervention set could be $\emptyset$.
For completeness, we give a short proof of \cref{lem:recovered-union} that generalizes the argument of \cite{ghassami2018budgeted} to non-atomic intervention sets.
In particular, our proof is much shorter because we use \cref{lem:triangle-lemma} in the case analysis of Meek R2.\footnote{For this case analysis, the proof of \cite[Appendix C]{ghassami2018budgeted} was more than 1.5 pages. Their structure $S_0$ (their figure 4) is precisely our \cref{lem:triangle-lemma} but it had some buggy arguments. For instance, in the case analysis of R2 with ground truth orientations $a \to c \to b \gets a$, they wish to argue that the arc $c \sim b$ would be oriented in $S_0$. However, their arguments concluded that $b \to c$ is oriented. Instead, they should use other arguments to conclude that $c \to b$ is oriented. For example, for the case analysis of $S_2$ (their figure 5), conditioned on $v_1 \to a$, Meek rules would enforce $v_1 \to b$ in the ground truth and thus $c \to b \gets v_1$ is a v-structure. They should have then used this to argue that $c \to b$ is oriented, instead of saying that Meek R4 orients $b \to c$. Fortunately, all such buggy arguments were fixable in their proofs and their conclusion is sound.}

\recoveredunion*
\begin{proof}
We show containment in both directions.

\textbf{Direction 1: $R(G,\cI_1) \cup R(G,\cI_2) \subseteq R(G, \cI_1 \cup \cI_2)$}

For any two interventions such that $A \subseteq B$, we can only recover more arc directions from the additional interventions in $B \setminus A$.
So, $R(G,\cI_1) \subseteq R(G, \cI_1 \cup \cI_2)$ and $R(G,\cI_2) \subseteq R(G, \cI_1 \cup \cI_2)$.

\textbf{Direction 2: $R(G, \cI_1 \cup \cI_2) \subseteq R(G,\cI_1) \cup R(G,\cI_2)$}

Consider an arbitrary edge $a \to b \in R(G, \cI_1 \cup \cI_2)$.
If $a \to b$ was oriented because there exists some intervention $S \in \cI_1 \cup \cI_2$ such that $|S \cap \{a,b\}| = 1$, then $a \to b \in R(G, \cI_1) \cup R(G,\cI_2)$ as well.
Suppose $a \to b$ was oriented in $R(G, \cI_1 \cup \cI_2)$ due to some Meek rule configuration (see \cref{fig:meek-rules}).
\begin{description}
    \item[R1] Suppose that $\exists c \in V$ such that $c \to a$ and $c \not\sim b$.
    That is, $c \to a$ is an oriented arc in either $R(G,\cI_1)$ and/or $R(G,\cI_2)$.
    Without loss of generality, $c \to a \in R(G,\cI_1)$.
    Then, R1 would have triggered and oriented $a \to b$ when we intervened on $\cI_1$ as well, i.e.\ $a \to b \in R(G,\cI_1)$.
    \item[R2] Suppose that $\exists$ $c \in V$ such that $a \to c \to b$.
    That is, $a \to c$ and $c \to b$ are oriented arcs in either $R(G,\cI_1)$ and/or $R(G,\cI_2)$.
    By \cref{lem:triangle-lemma}, it \emph{cannot} be the case $R(G,\cI_1)$ or $R(G,\cI_2)$ contains only exactly one of these arcs.
    Without loss of generality, suppose that $a \to c, c \to b \in R(G,\cI_1)$.
    Then, R2 would have triggered and oriented $a \to b$ when we intervened on $\cI_1$ as well, i.e.\ $a \to b \in R(G,\cI_2)$.
    \item[R3] Suppose that $\exists$ $c,d \in V$ such that $d \sim a \sim c$, $d \to b \gets c$, and $c \not\sim d$.
    Since $d \to b \gets c$ is a v-structure, it will appear in the observational essential graph and R3 will trigger to orient $a \to b$ in both $R(G,\cI_1)$ and $R(G,\cI_2)$.
    \item[R4] Suppose that $\exists$ $c,d \in V$ such that $d \sim a \sim c$, $d \to c \to b$, and $b \not\sim d$.
    That is, $d \to c$ and $c \to b$ are oriented arcs in either $R(G,\cI_1)$ and/or $R(G,\cI_2)$.
    Without loss of generality, $c \to a \in R(G,\cI_1)$.
    Then, when we intervened on $\cI_1$, R1 would have triggered to orient $c \to b$, and then R4 would have triggered to orient $a \to b$, i.e.\ $c \to b, a \to b \in R(G,\cI_1)$.
\end{description}
In all cases, we see that $a \to b \in R(G,\cI_1) \cup R(G,\cI_2)$.
\end{proof}
\section{Deferred proofs}
\label{sec:proofs}

\subsection{Properties of interventional essential graphs}

Our proof of \cref{thm:properties} is greatly simplified by \cref{lem:triangle-lemma}, an observation\footnote{A similar argument was made in \cite[Appendix B, Figure 4, Structure $S_0$]{ghassami2018budgeted} for their case analysis proof of \cref{lem:recovered-union}.} that triangles in interventional essential graphs \emph{cannot} have exactly one oriented arc.
The proof of \cref{lem:triangle-lemma} relies on the following known fact:

\begin{lemma}[Proposition 15 of \cite{hauser2012characterization}]
\label{lem:essential-graph-is-chain-graph}
Consider the $\cI$-essential graph $\cE_{\cI}(G^*)$ of some DAG $G^*$ and let $H \in CC(\cE_{\cI}(G^*))$ be one of its chain components.
Then, $\cE_{\cI}(G^*)$ is a chain graph and $\cE_{\cI}(G^*)[V(H)]$ is chordal.
\end{lemma}

Recall that chain graphs are partially directed graphs that do not contain directed cycles (i.e.\ a sequence of edges forming an undirected cycle with at least one oriented arc, and all oriented arcs are in the same direction along this cycle).

\begin{restatable}[Triangle lemma]{lemma}{trianglelemma}
\label{lem:triangle-lemma}
Consider a DAG $G = (V,E)$ and an intervention set $\cI \subseteq 2^V$.
For any triangle on vertices $u,v,w \in V$ and three edges $u \sim v, v \sim w, u \sim w \in E$, we have $|R(G,\cI) \cap \{u \sim v, v \sim w, u \sim w\}| \neq 1$.
\end{restatable}
\begin{proof}
Recall that $R(G, \cI) = A(\cE_{\cI}(G))$ and $\cE_{\cI}(G)$ is a chain graph (\cref{lem:essential-graph-is-chain-graph}).
Suppose there is a triangle on $u,v,w$ and $|R(G,\cI) \cap \{u \sim v, v \sim w, u \sim w\}| = 1$.
Without loss of generality, suppose $u \to v \in R(G,\cI)$.
Then, $u \to v \sim w \sim u$ is a directed cycle in $\cE_{\cI}(G)$, contradicting the fact that $\cE_{\cI}(G)$ is a chain graph.
\end{proof}

Note that there exists partially oriented chain graphs that are \emph{not} interventional essential graphs where every triangle does not have exactly one oriented arc, and the edge-induced subgraph on the unoriented arcs do not form v-structures for any acyclic completion.
See \cref{fig:converse-counterexample}.

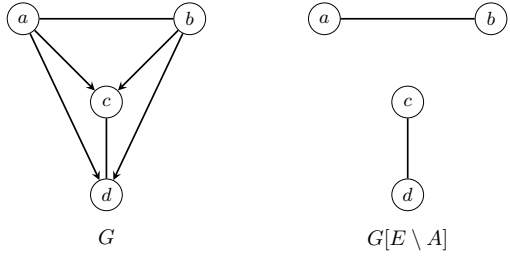
\begin{figure}[htbp]
\centering
\resizebox{0.4\linewidth}{!}{%
\begin{tikzpicture}
\node[draw, circle, minimum size=15pt, inner sep=2pt] at (0,0) (Gc) {\small $c$};
\node[draw, circle, minimum size=15pt, inner sep=2pt, above left=of Gc] (Ga) {\small $a$};
\node[draw, circle, minimum size=15pt, inner sep=2pt, above right=of Gc] (Gb) {\small $b$};
\node[draw, circle, minimum size=15pt, inner sep=2pt, below=of Gc] (Gd) {\small $d$};
\draw[thick, -stealth] (Ga) -- (Gc);
\draw[thick, -stealth] (Gb) -- (Gc);
\draw[thick, -stealth] (Ga) -- (Gd);
\draw[thick, -stealth] (Gb) -- (Gd);
\draw[thick] (Ga) -- (Gb);
\draw[thick] (Gc) -- (Gd);
\node[text centered, below=5pt of Gd] {$G$};

\node[draw, circle, minimum size=15pt, inner sep=2pt] at (5,0) (Hc) {\small $c$};
\node[draw, circle, minimum size=15pt, inner sep=2pt, above left=of Hc] (Ha) {\small $a$};
\node[draw, circle, minimum size=15pt, inner sep=2pt, above right=of Hc] (Hb) {\small $b$};
\node[draw, circle, minimum size=15pt, inner sep=2pt, below=of Hc] (Hd) {\small $d$};
\draw[thick] (Ha) -- (Hb);
\draw[thick] (Hc) -- (Hd);
\node[text centered, below=5pt of Hd] {$G[E \setminus A]$};
\end{tikzpicture}
}
\caption{
In the partially oriented chain graph $G$ with oriented arcs $A$, all triangles have exactly two oriented arcs.
Since $a \sim b$ and $c \sim d$ could be independently oriented in either directions, there are four possible acyclic completions of $G$.
The edge-induced subgraph $G[E \setminus A]$ of the unoriented arcs does not have any v-structures for any of these possible acyclic completions.
However, $G$ \emph{cannot} be an interventional essential graph: there are no v-structures and every vertex is incident to some unoriented edge.
}
\label{fig:converse-counterexample}
\end{figure}

Recall the definition of oriented subgraphs and recovered parents:
For any interventional set $\cI \subseteq 2^V$ and $u \in V$, $G^{\cI} = G[E \setminus R(G,\cI)]$ is the \emph{oriented} subgraph induced by the \emph{unoriented arcs} in $G$ and $\Pa_{G,\cI}(u) = \{x \in V: x \to u \in R(G,\cI)\}$ is the recovered parents of $u$ by $\cI$.

\interventionalessentialgraphproperties*
\begin{proof}\hspace{0pt}
\begin{enumerate}
    \item By definition of $G^{\cA}$ and chain components.
    
    \item For the statement to be false, there must exist a triangle in $G$ on 3 vertices $u,v,w$ such that $u \to v \in R(G,\cA)$ and $u \to w, v \to w \not\in R(G,\cA)$.
    This is impossible by \Cref{lem:triangle-lemma}.
    
    \item Without loss of generality, it suffices to consider two adjacent vertices $u$ and $v$ in the same chain component of $\cE_{\cA}(G)$ with $|\Pa_{G,\cA}(u)| \geq |\Pa_{G,\cA}(v)|$.
    This is because for the claim to hold between any two vertices $x$ and $y$ in the same chain component, we can apply the result for consecutive pairs of adjacent vertices between any connected path between $x$ and $y$.
    
    If $|\Pa_{G,\cA}(u)| = 0$, then the claim trivially holds since $|\Pa_{G,\cA}(v)| \leq |\Pa_{G,\cA}(u)| = 0$.
    
    Now, for $|\Pa_{G,\cA}(u)| > 0$, consider an arbitrary vertex $x \in \Pa_{G,\cA}(u)$.
    If $x \not\sim v$ in $G$, then $u \to v \in R(G,\cA)$ via R1 configuration $x \to u \sim v$, thus $u$ and $v$ are \emph{not} adjacent in the chain components of $\cE_{\cA}(G)$.
    Meanwhile, \cref{lem:triangle-lemma} tells us we \emph{cannot} have $x \sim v$ in $G$ with $x \sim v \not\in R(G,\cA)$, as this would further imply $|R(G,\cA) \cap \{x \sim u, x \sim v, u \sim v\}| = 1$, which is a contradiction.
    If $v \to x$ in $G$ and $v \to x \in R(G,\cA)$, then $u \to v \in R(G,\cA)$ via R2 configuration $v \to x \to u \sim v$, thus $u$ and $v$ are \emph{not} adjacent in the chain components of $\cE_{\cA}(G)$.
    So, we must have $x \to v$ in $G$ and $x \to v \in R(G,\cA)$.
    That is, $x \in \Pa_{G,\cA}(v)$.
    
    \item Suppose, for a contradiction, that $u$ and $v$ lie in the same chain component of $\cE_{\cA}(G)$.
    Since $u \to v \in R(G,\cA)$, we see that $u \notin \Pa_{G,\cA}(u)$ and $u \in \Pa_{G,\cA}(v)$, i.e.\ $\Pa_{G,\cA}(u) \neq \Pa_{G,\cA}(v)$.
    This is a contradiction to condition 3.
    
    \item
    Fix an acyclic completion $G'$ of $\cE(G^{\cA})$.
    Suppose, for a contradiction, that there is a cycle in $E(G') \cup R(G,\cA)$.
    Let $C = v_0 \to v_1 \to \ldots \to v_k \to v_0$ be the \emph{smallest} such cycle.
    Since $G'$ is an acyclic completion, we know that at least one arc of $C$ was from $R(G,\cA)$.
    Without loss of generality, suppose that $v_0 \to v_1 \in R(G,\cA)$.
    Since $G^*$ is acyclic, we also know that at least one arc of $C$ is \emph{not} from $R(G,\cA)$.
    
    If $k = 2$, then $C = v_0 \to v_1 \to v_2 \to v_0$.
    By \cref{lem:triangle-lemma}, we \emph{cannot} have $v_1 \to v_2, v_2 \to v_0 \not\in R(G,\cA)$.
    \begin{itemize}
        \item If $v_0 \to v_1, v_1 \to v_2 \in R(G,\cA)$ and $v_2 \to v_0 \not\in R(G,\cA)$, then Meek rule R2 will orient $v_0 \to v_2$ via $v_0 \to v_1 \to v_2 \sim v_0$.
        \item If $v_2 \to v_0, v_0 \to v_1 \in R(G,\cA)$ and $v_1 \to v_2 \not\in R(G,\cA)$, then Meek rule R2 will orient $v_2 \to v_1$ via $v_2 \to v_0 \to v_1 \sim v_2$.
    \end{itemize}
    In any case, we arrive at a contradiction.
    
    Now, consider the case where $k > 2$.
    Let $v_i \to v_j \not\in R(G,\cA)$ be the arc of $C$ with the smallest source index $i \geq 1$, where we write $j = (i+1) \mod k$ for notational convenience.
    Since $v_i \to v_j \not\in R(G,\cA)$, it must be the case that the arc $v_{i-1} \sim v_j$ exists in $G$, otherwise Meek rule R1 will orient $v_i \to v_j$ via $v_{i-1} \to v_i \sim v_j$.
    By \cref{lem:triangle-lemma} and the assumption that $v_i \to v_j \not\in R(G,\cA)$, it must be the case that $v_{i-1} \sim v_j$ is oriented in $R(G,\cA)$.
    \begin{itemize}
        \item If $v_{i-1} \to v_j \in R(G,\cA)$, then $v_0 \to \ldots \to v_{i-1} \to v_j \to \ldots \to v_k \to v_0$ is a smaller cycle than $C$ in $E(G') \cup R(G,\cA)$.
        \item If $v_j \to v_{i-1} \in R(G,\cA)$, then Meek rule R2 orients $v_j \to v_i$ via $v_j \to v_{i-1} \to v_i \sim v_j$.
    \end{itemize}
    In either case, we arrive at a contradiction.
    
    \item We show containment in both directions.
    
    \textbf{Direction 1: $R(G^{\cA},\cB) \subseteq R(G,\cB) \setminus R(G,\cA)$}
    
    Suppose, for a contradiction, that there exists an arc $a \to b \in R(G^{\cA},\cB)$ but $a \to b \not\in R(G,\cB) \setminus R(G,\cA)$.
    Note that $a \to b \not\in R(G,\cA)$ otherwise $a \to b$ will \emph{not} appear in $G^{\cA}$ and thus cannot be in $R(G^{\cA},\cB)$.
    So, to show a contradiction, it suffices to argue that $a \to b \in R(G,\cB)$.
    
    There are two possible situation explaining $a \to b \in R(G^{\cA},\cB)$: either (i) there is some intervention $I \in \cB$ such that $|I \cap \{a,b\}| = 1$, or (ii) Meek rules oriented $\{a,b\}$.
    
    (i) In the first situation where there is some intervention $I \in \cB$ such that $|I \cap \{a,b\}| = 1$, we see that $a \to b \in R(G,\cB)$ as well.
    Contradiction.
    
    (ii) In the second situation, let us consider the sequence of Meek rule configurations that oriented $a \to b$ in $R(G^{\cA},\cB)$.
    By definition of $R(G^{\cA},\cB)$, all the edges (oriented or not) involved in these configurations do \emph{not} belong to $R(G,\cA)$.
    If these configurations also appear in $R(G,\cB)$, then $a \to b \in R(G,\cB)$ as well.
    The only reason why any of these configurations may not appear in $R(G,\cB)$ is because there was some other edge in the node-induced subgraph that was removed due to being in $R(G,\cA)$:
    \begin{itemize}
        \item Suppose the R1 configuration involving three vertices $u \to v \sim w$ and $u \not\sim w$ was one of the configurations used by $R(G^{\cA},\cB)$ to orient $a \to b$, but this configuration \emph{did not} appear for $R(G,\cB)$.
        Then, it was because $u \sim w$ appears in $G$ and was removed from $G^{\cA}$ due to it being oriented in $R(G,\cA)$.
        However, this means that $|R(G,\cA) \cap \{u,v,w\}| = 1$, contradicting \cref{lem:triangle-lemma}.
        \item All possible edges are present in the node-induced subgraph of the R2 configuration.
        \item There is only one possible edge removed by $R(G,\cA)$ in configurations R3 and R4.
        By the same argument to the R1 configuration above, one can check that this implies that there is some triangle on three vertices $u,v,w$ within the configuration such that $|R(G,\cA) \cap \{u,v,w\}| = 1$, contradicting \cref{lem:triangle-lemma}.
    \end{itemize}
    In other words, whenever $a \to b \in R(G^{\cA},\cB)$ due to Meek rules, we see that $a \to b \in R(G,\cB)$.
    
    \textbf{Direction 2: $R(G,\cB) \setminus R(G,\cA) \subseteq R(G^{\cA},\cB)$}
    
    For any arc $a \to b \in R(G,\cB) \setminus R(G,\cA)$, we have that $a \to b \not\in R(G,\cA)$ and so the edge $a \sim b$ appears in $G^{\cA}$.
    That is, we may ignore v-structure arcs in $R(G,\cB)$.
    There are two possible situation explaining why an arc $a \to b$ belongs in $R(G,\cB)$: either (i) there is some intervention $I \in \cB$ such that $|I \cap \{a,b\}| = 1$, or (ii) Meek rules oriented $\{a,b\}$.
    
    (i) In the first situation, we have $a \to b \in R(G^{\cA},\cB)$ as well.
    
    (ii) We prove the second situation by contradiction.
    Suppose, for a contradiction, that $(R(G,\cB) \setminus R(G,\cA)) \setminus R(G^{\cA},\cB)$ is non-empty.
    Let $a \to b \in (R(G,\cB) \setminus R(G,\cA)) \setminus R(G^{\cA},\cB)$ be oriented in $R(G,\cB)$ via a sequence of Meek rule configurations such that only the last configuration does not appear in $R(G^{\cA},\cB)$.
    By calling such a Meek rule configuration a \emph{bad} configuration, we can see why such an arc $a \to b$ exists: for any arc in $(R(G,\cB) \setminus R(G,\cA)) \setminus R(G^{\cA},\cB)$ that uses more than one bad configuration, one of the oriented arcs in the bad configuration is an arc in $(R(G,\cB) \setminus R(G,\cA)) \setminus R(G^{\cA},\cB)$ that is oriented with strictly fewer bad orientations.
    
    Now, consider the last Meek rule configuration used to orient $a \to b$ in $R(G,\cB)$.
    We make two observations:
    \begin{description}
        \item[O1] If \emph{none} of the oriented arcs of this Meek rule configuration belongs to $R(G,\cA)$, then these arcs appear in $G^{\cA}$ and will be oriented due to $B$, thus $a \to b \in R(G^{\cA},\cB)$.
        \item[O2] If \emph{all} of the oriented arcs of this Meek rule configuration belong to $R(G,\cA)$, then $a \to b \in R(G,\cA)$, which is a contradiction.
    \end{description}
    There is only one arc in the R1 configuration, so either O1 or O2 applies.
    Meanwhile, the arcs in the R3 configuration form a v-structure and so \emph{both} of them belong to $R(G,\{\emptyset\}) \subseteq R(G,\cA)$, so O2 applies.
    In R2 or R4 configurations, there are two arcs.
    If none or both arcs are in $R(G,\cA)$, then we can apply O1 or O2.
    If exactly one of the arcs are in $R(G,\cA)$, then there will be a triangle on three vertices $u,v,w$ within the configuration such that $|R(G,\cA) \cap \{u,v,w\}| = 1$.
    This is impossible according to \cref{lem:triangle-lemma}.
    Since all cases except O1 lead to contradiction, we must have $a \to b \in R(G^{\cA},\cB)$.
    This contradicts our assumption that $a \to b \in (R(G,\cB) \setminus R(G,\cA)) \setminus R(G^{\cA},\cB)$.
    
    \item By \cref{lem:recovered-union}, we have that $R(G, \cA \cup \cB) = R(G,\cA) \cup R(G,\cB)$.
    The claim follows using statement 6.
    
    \item The disjointness follows from definitions of $G^{\cA}$ and $G^{\cB}$.
    We now argue containment in both directions.
    
    \textbf{Direction 1: $R(G,\cA \cup \cB) \subseteq R(G^{\cA},\cB) \;\dot\cup\ R(G^{\cB},\cA) \;\dot\cup\; (R(G,\cA) \cap R(G,\cB))$}
    
    By \cref{lem:recovered-union}, we know that $R(G,\cA \cup \cB) = R(G,\cA) \cup R(G,\cB)$.
    Consider an arbitrary arc $e \in E$ such that $e \in R(G,\cA) \cup R(G,\cB)$.
    Suppose $e \not\in R(G,\cA) \cap R(G,\cB)$.
    If $e \in R(G,\cA) \setminus R(G,\cB)$, then $e$ appears in $G^{\cB}$ and so $e \in R(G^{\cB},\cA)$.
    If $e \in R(G,\cB) \setminus R(G,\cA)$, then $e$ appears in $G^{\cA}$ and so $e \in R(G^{\cA},\cB)$.
    In either case, we see that $e \in R(G^{\cB},\cA) \cup R(G^{\cA},\cB) \subseteq R(G^{\cA},\cB)$.
    Therefore, $e \in R(G^{\cA},\cB) \;\dot\cup\ R(G^{\cB},\cA) \;\dot\cup\; (R(G,\cA) \cap R(G,\cB))$.
    
    \textbf{Direction 2: $R(G^{\cA},\cB) \;\dot\cup\ R(G^{\cB},\cA) \;\dot\cup\; (R(G,\cA) \cap R(G,\cB)) \subseteq R(G,\cA \cup \cB)$}
    
    We argue that each of $R(G^{\cA},\cB)$, $R(G^{\cB},\cA)$, and $R(G,\cA) \cap R(G,\cB)$ is a subset of $R(G,\cA \cup \cB)$.
    By statement 7, $R(G^{\cA},\cB) \subseteq R(G,\cA \cup \cB)$ and $R(G^{\cB},\cA) \subseteq R(G,\cA \cup \cB)$.
    By \cref{lem:recovered-union}, we know that $R(G,\cA \cup \cB) = R(G,\cA) \cup R(G,\cB)$ and so $R(G,\cA) \cap R(G,\cB) \subseteq R(G,\cA) \cup R(G,\cB) \subseteq R(G,\cA \cup \cB)$.
    
    \item By definition, covered edges are not v-structure edges.
    By \cite[Lemma 27]{choo2022verification}, covered edges will not be oriented by Meek rules and we need to intervene on either of the endpoints to orient it.
    Therefore, $R(G, \emptyset)$ does not contain any covered edges.
    \qedhere
\end{enumerate}
\end{proof}

\subsection{Hasse diagrams of DAGs without v-structures}

\vstructfreeiffHasseisatree*
\begin{proof}
We prove each direction separately.

\textbf{Direction 1: If DAG $G$ is a single connected component without v-structures, then the Hasse diagram $H_G$ is a directed tree with a unique root vertex}

Suppose, for a contradiction, that there are two distinct paths $P_1 = (u, \ldots, u', x)$ and $P_2 = (v, \ldots, v', x)$ in $H_G$ that end at some vertex $x \in V$, where $u' \neq v'$.
If $u' \not\sim v'$ in $G$, then $u' \to x \gets v'$ is a v-structure.
Without loss of generality, $u' \to v'$.
But this means that $x \not\in \Ch(u')$ and so we should not have an arc $u' \to x$ in the Hasse diagram $H_G$.
Contradiction.

\textbf{Direction 2: If the Hasse diagram $H_G$ is a directed tree with a unique root vertex, then DAG $G$ is a single connected component without v-structures.}

Suppose, for a contradiction, that the DAG $G$ has a v-structure $u \to x \gets v$.
Since $u,v \in \Anc(x)$ and reachability is preserved in Hasse diagrams, there will be paths $P_1 = (u, \ldots, u', x)$ and $P_2(v, \ldots, v', x)$ in $H_G$.
Since $H_G$ has a unique root, $u'$ and $v'$ must have a common ancestor $y$ in $H_G$ ($y$ could be the root itself).
But this means that the Hasse diagram is \emph{not} a directed tree since there are two paths from $y$ to $x$ in $H_G$.
Contradiction.
\end{proof}

The proof of \cref{thm:orienting-vertices-form-an-interval} relies on \cref{lem:intermediate-direct-arcs-exist}, \cref{lem:middle}, \cref{cor:child-orients-all-outgoing arcs}, \cref{lem:if-w-orients-then-y-orients}, and \cref{lem:if-w-does-not-orient-then-x-does-not-orient}, which we prove first.

\begin{restatable}{lemma}{intermediatedirectarcsexist}
\label{lem:intermediate-direct-arcs-exist}
Suppose $H_G$ is a rooted tree induced by a DAG $G = (V,E)$ without v-structures.
If $u \to v$ in $G$, then $u \to w$ in $G$ for any two vertices $u,v \in V$ and for all $w \in \Des(u) \cap \Anc(v)$.
\end{restatable}
\begin{proof}
If $u \to v$ in $G$, then there exists a path $P_{u \to v}$ in $H_G$.
If $P_{u \to v} = (u,v)$ is a direct arc, then the claim is vacuously true.
Suppose $P_{u \to v} = (u, w_1, \ldots, w_k, v)$ where $\Des(u) \cap \Anc(v) = \{w_1, \ldots, w_k\}$.
This implies that the arcs $u \to w_1 \to \ldots \to w_k \to v$ are all present in $G$.
Since $G$ has no v-structures, it must be the case that the arc $u \to w_k$ exists (otherwise $u \to v \gets w_k$ is a v-structure).
Thus, by recursive argument from $w_{k-1}$ up to $w_1$, there must be arcs $u \to w$ in $G$ for \emph{any} $w \in \{w_1, \ldots, w_k\}$.
\end{proof}

\begin{restatable}{lemma}{middle}
\label{lem:middle}
Let $G = (V,E)$ be a DAG without v-structures and $u \to v$ be an unoriented arc in $\cE(G)$.
Then, we have $w \in R^{-1}_1(G, u \to v)$ for any $w \in \Des(u) \cap \Anc(v)$ in the Hasse diagram $H_G$.
\end{restatable}
\begin{proof}
If $v \in \Ch(u)$, then the result is vacuously true since $\Des(u) \cap \Anc(v) = \emptyset$.
Suppose $P_{u \to v} = (u, w_1, \ldots, w_k, v)$ is the unique path from $u$ to $v$ in $H_G$, where $\Des(u) \cap \Anc(v) = \{w_1, \ldots, w_k\}$.
By \cref{lem:intermediate-direct-arcs-exist}, we know that the arc $u \to w$ exists in $G$ for any $w \in \Anc(v) \cap \Des(u)$.
Suppose we intervened on an arbitrary $w_i \in \{w_1, \ldots w_k\}$, where $\Des(w_i) \cap \Anc(v) = \{w_{i+1}, \ldots, w_k, v\}$.
For any fixed arbitrary valid permutation $\pi$, define
\[
\last(\pi, w_i) = \argmax_{\substack{z \in \Des(w_i) \cap \Anc(v)\\(w_i \to z) \in E}} \{\pi(z)\}
\]
as the ``last'' vertex in $\Des(w_i) \cap \Anc(v)$ that $w_i$ has a direct arc within $G $.

If $\last(\pi, w_i) = v$, then intervening on $w_i$ yields $u \to w_i \to w_j = v \sim u$.
So, Meek rule R2 will trigger to orient $u \to v$.
Otherwise, if $\last(\pi, w_i) \neq v$, then intervening on $w_i$ will cause two sets of Meek rules to fire:
(1) Meek rule R2 will orient the arcs $u \to w_j$, for all $j \in \Des(w_i) \cap \Anc[\last(\pi, w_i)]$ since $u \to w_i \to w_j \sim u$;
(2) Meek rule R1 will orient all outgoing arcs of $\last(\pi, w_i)$, since $w_i \not\sim w_z$ for all $z \in \Des(\last(\pi, w_i))$, by maximality of $\last(\pi, w_i)$.
Repeating the above argument by replacing the role of $w_i$ by $\last(\pi, w_i)$, we see that the arc $w_{final} \to v$ will eventually be oriented by some $w_{final} \in \Des(w_i)$, and so Meek rule R2 will orient $u \to v$.
Intuitively, the direction $u \to v$ is forced in order to avoid a directed cycle since we will have $u \to w_i \to \last(\pi, w_i) \to \last(\last(\pi, w_i)) \ldots \to \last(\last(\ldots(\last(\pi, w_i)))) = w_{final} \to v \sim u$.
\end{proof}

\begin{restatable}{corollary}{childorientsalloutgoingarcs}
\label{cor:child-orients-all-outgoing arcs}
Let $G = (V,E)$ be a DAG without v-structures.
For a vertex $w$ and a direct child $y \in \Ch(w)$, we have $\{w \to z : z \in V(T_y)\} \subseteq R_1(G,y)$ where $T_y$ is the subtree rooted at $y$ in the Hasse diagram $H_G$.
That is, intervening on $y$ orients all outgoing arcs of $w$ with an endpoint in $T_y$.
\end{restatable}
\begin{proof}
Since $y \in \Des(w) \cap \Anc(z_i)$ for any $z_i \in V(T_y)$, \cref{lem:intermediate-direct-arcs-exist} gives $w \to z_i \in R_1(G,y)$.
\end{proof}

\begin{restatable}{lemma}{ifworientsthenyorients}
\label{lem:if-w-orients-then-y-orients}
Let $G = (V,E)$ be a DAG without v-structures and $u \to v$ be an unoriented arc in $\cE(G)$.
If a vertex $w \in R^{-1}_1(G, u \to v)$, then $y \in R^{-1}_1(G, u \to v)$ for all $y \in \Des(w) \cap \Anc(u)$.
\end{restatable}
\begin{proof}
We begin by making two observations which grants us stronger properties about $w$ and $y$:
\begin{enumerate}
    \item If $w \in \{u,v\}$, then $\Des(w) \cap \Anc(u) = \emptyset$ and the result is trivially true.
    \item Suppose the chain of direct children from $w$ to $v$ is $w \to y_1 \to y_2 \to \ldots \to y_k \to u \to v$.
    To prove the result, it suffices to argue that $y_1 \in R^{-1}_1(G, u \to v)$ and then apply induction to conclude that $y_2 \in R^{-1}_1(G, u \to v)$, and so on.
\end{enumerate}
Thus, in the rest of the proof, we can assume that $w \not\in \{u,v\}$ and $y \in \Ch(w)$ is a direct child of $w$.

Since $w \not\in \{u,v\}$ and $y \in \Ch(w)$, we see that the arc $u \to v$ belongs in the set $R_1(G, w) \cap B(T_y)$, where $T_y$ is the subtree rooted at $y$ in the Hasse diagram $H_G$.
So, it suffices to show that $R_1(G, w) \cap B(T_y) \subseteq R_1(G, y)$.
By \cref{lem:if-v-orients-then-v-to-b-path-exists}, intervening on $w$ will \emph{not} orient new arc directions of the form $a \to b$ where $b \in \Anc(w)$.
So, we can partition the \emph{newly recovered arcs} in $R_1(G, w)$ into three disjoint sets $R^1(w)$, $R^2(w)$, and $R^3(w)$ as follows:
\begin{align*}
R^1(w) &= \{ a \to w : a \in \Anc(w) \}\\
R^2(w) &= \{ a \to b : a \in \Anc[w], b \in \Des(w) \}\\
R^3(w) &= \{ a \to b : a, b \in \Des(w) \}
\end{align*}

Clearly, $R^1(w) \cap B(T_y) = \emptyset$ since neither endpoint lies in $T_y$.
By \cref{lem:intermediate-direct-arcs-exist} and \cref{cor:child-orients-all-outgoing arcs}, we know that $R^2(w) \cap B(T_y) \subseteq R_1(G, y)$.
Thus, it suffices to argue that $R^3(w) \cap B(T_y) \subseteq R_1(G, y)$.
To do so, consider an arbitrary (non-unique) ordering $\sigma$ on the arcs in $R^3(w) \cap B(T_y)$ by which Meek rule orients them: $\sigma(a \to b) < \sigma(c \to d)$ means that the arc $a \to b$ was oriented before $c \to d$.

Suppose, for a contradiction, that there exists some arc in $(R^3(w) \cap B(T_y)) \setminus R(y)$.
Let $a \to b \in (R^3(w) \cap B(T_y)) \setminus R(y)$ be the arc with the \emph{minimal} $\sigma$ ordering, where $a,b \in \Des[y]$.
We check the four Meek rule configurations that could have oriented $a \to b$ while using some arc orientation that is \emph{not} in $R_1(G, y)$.
If the oriented arc in the configuration belongs to $R^3(w)$, then it must be arc with lower $\sigma$ ordering than $a \to b$ and is oriented in $R_1(G, y)$ by assumption.
Meanwhile, observe that any arc in $R^1(w) \cup R^2(w)$ has $\Anc[w]$ as the start of the arc.
In the configurations of Meek rule R1 and R2, if any of the oriented arcs belong to $R^1(w) \cup R^2(w)$, then $a \to b \not\in B(T_y)$ since $a \in \Anc[w]$.
So, it suffices to check only Meek rules R3 and R4:
\begin{description}
    \item[R3] There exists $c,d$ such that $d \sim a \sim c$, $d \to b \gets c$, $c \not\sim d$.\\
    If $(c \to b) \in R^1(w)$ or $d \to b \in R^1(w)$, then $b = w$ and $a \to b \not\in B(T_y)$.
    Meanwhile, if $c \to b \in R^2(w)$ or $d \to b \in R^2(w)$, then $c \to b, d \to b \in R_1(G, y)$ by \cref{lem:middle}.
    Thus, Meek rule R3 will trigger and $a \to b \in R_1(G, y)$.
    \item[R4] There exists $c,d$ such that $d \sim a \sim c$, $d \to c \to b$, $b \not\sim d$.\\
    If $d \in \Des(w)$, then all the arcs belong to $R^3(w)$ and are thus trivially oriented.
    Suppose now that $d \in \Anc[w]$.
    If $c \to b \in R^1(w)$, then $b = w$ and $a \to b \not\in B(T_y)$.
    If $a = y$, then $a \to b \in R_1(G, y)$ trivially.
    Otherwise, $a \in \Des(y)$ and $d \to a$ will be oriented by \cref{lem:middle} since $d \in \Anc[w]$.
    Then, we have $d \to a \sim b$ and Meek R1 will trigger and orient $a \to b$.
    In other words, $a \to b \in R_1(G, y)$.
\end{description}
Since we always conclude that $a \to b \in R_1(G, y)$, this is a contradiction.
\end{proof}

\begin{restatable}{lemma}{ifwdoesnotorientthenxdoesnotorient}
\label{lem:if-w-does-not-orient-then-x-does-not-orient}
Let $G = (V,E)$ be a DAG without v-structures and $u \to v$ be an unoriented arc in $\cE(G)$.
If a vertex $w \not\in R^{-1}_1(G, u \to v)$, then $x \not\in R^{-1}_1(G, u \to v)$ for all $x \in \Anc(w)$.
\end{restatable}
\begin{proof}
Suppose, for a contradiction, that $x \in R^{-1}_1(G, u \to v)$.
Since $x \in \Anc(w)$, we have that $x \in \Anc(u)$ and $w \in \Des(x) \cap \Anc(u)$.
By \cref{lem:if-w-orients-then-y-orients}, $w \in R^{-1}_1(G, u \to v)$.
Contradiction.
\end{proof}

We are now ready to prove \cref{thm:orienting-vertices-form-an-interval}.

\orientingverticesformaninterval*
\begin{proof}
We have $u,v \in R^{-1}_1(G, u \to v)$ trivially.
By \cref{lem:if-v-orients-then-v-to-b-path-exists}, $\Des(v) \cap R^{-1}_1(G, u \to v) = \emptyset$.
By \cref{lem:if-v-orients-then-v-to-b-path-exists}, $R^{-1}_1(G, u \to v) \subseteq \Anc[v]$.
For an arbitrary consistent topological ordering $\pi$, let
\[
w = \argmin_{\substack{z \in \Anc(u)\\z \in R^{-1}_1(G, u \to v)}} \{\pi(z)\}
\]
be the ``furthest'' ancestor vertex of $u$ that orients $u \to v$.
By \cref{lem:if-w-orients-then-y-orients}, $\Des(w) \cap \Anc(u) \subseteq R^{-1}_1(G, u \to v)$.
By minimality of $w$ and \cref{lem:if-w-does-not-orient-then-x-does-not-orient}, $\Anc(w) \cap R^{-1}_1(G, u \to v) = \emptyset$.
Putting everything together, we see that $R^{-1}_1(G, u \to v) = \Des[w] \cap \Anc[v]$.
\end{proof}

\coverededgesareHasseedges*
\begin{proof}
To prove this, we argue that any edge $a \to b \not\in E(H_G)$ \emph{cannot} be a covered edge.
Since $E(H_G)$ only contains arcs involving direct children, we see that $b \not\in \Ch(a)$.
So, there exists some $z \in \Des(a) \cap \Anc(b)$ such that $z \to b$ but $z \not\to a$.
Thus, $a \to b$ \emph{cannot} be a covered edge.
\end{proof}

\subsection{Subset verification with atomic interventions}

\subsetverificationashardasintervalstabbing*
\begin{proof}
By \cref{thm:orienting-vertices-form-an-interval}, we know that each target edge $e \in T$ has a corresponding interval $[a_e, b_e]_{H}$ will be oriented if and only if some vertex in $[a_e, b_e]_{H}$ is selected into the intervention set.
Define $\cJ = \{[a_e, b_e] : e \in T\}$ as the collection of intervals corresponding to each edge $e \in T$ of the target edges.
Then, any solution to the interval stabbing problem on $(H,\cJ)$ ensures that every interval is stabbed, which translates to every edge in $T$ being oriented via \cref{thm:orienting-vertices-form-an-interval}.
Meanwhile, the minimality of the interval stabbing solution corresponds to the minimality of the atomic verification set size.
\end{proof}

\existspolytimealgoforintervalstabbing*
\begin{proof}
See \cref{thm:poly-time-for-interval-stabbing}.
\end{proof}

\existspolytimealgoforsubsetverification*
\begin{proof}
Since closure under Meek rules can be computed in polynomial time (e.g.\ via \cite[Algorithm 2]{pmlr-v161-wienobst21a}), we can compute all $R(G,v)$ for each $v \in V$, and thus $R^{-1}(u \to v)$ in polynomial time.
Then, the reduction given in \cref{lem:subset-verification-as-hard-as-interval-stabbing} runs in polynomial time and we can apply the polynomial time algorithm of \cref{lem:exists-poly-time-algo-for-interval-stabbing} to solve the resulting interval stabbing instance.
\end{proof}

\intervalstabbingashardassubsetverification*
\begin{proof}
Consider the following construction:
\begin{enumerate}
    \item Relabel endpoints $(u,v) \in \cJ$ such that $\pi(u) < \pi(v)$, if necessary.
    \item Define $E' = E \cup \cJ \cup A$, where $A$ is the set of additional arcs defined as follows: For each $(u,v) \in \cJ$, add $z \to w$ for all $z \in \Anc(v)$ and $w \in \Des[u] \cap \Anc[v]$.
    \item Let $G = (V, E')$ be the resulting DAG and let $T = \{u \to v: (u,v) \in \cJ\}$.
        Note that $G$ is a DAG without v-structures and one can check that the Hasse diagram is exactly equal to $H$.
\end{enumerate}

To argue that the solution to the subset verification problem instance $(G,T)$ is a solution to the interval stabbing on a tree instance $(H,\cJ)$, it suffices to show that $R^{-1}_1(G, u \to v) = \Des[u] \cap \Anc[v]$ for each arc $u \to v \in T$.

Consider an arbitrary $u \to v \in T$.
By \cref{lem:recovered-union}, it suffices to consider an arbitrary vertex $w$ in the atomic intervention set.
By \cref{lem:if-v-orients-then-v-to-b-path-exists}, we know that $u \to v \not\in R(G,w)$ if $w \cap Anc[v] $.
We also know from \cref{lem:middle} that $u \to v \in R(G,w)$ if $w \in \Des[u] \cap \Anc[v]$.
It remains to argue that $u \to v \not\in R(G,w)$ for $w \in \Anc(u)$.
We do this by arguing that Meek rules \emph{cannot} orient $u \to v$ through an intervention on any $w \in \Anc(u)$.
\begin{description}
    \item[R1] Meek rule R1 cannot trigger to orient $u \to v$ since the arc $w \to v \in E'$ whenever $w \to u$ exists, by construction.
    \item[R3] Since $G$ is a DAG without v-structures, Meek rule R3 will never be invoked.
    \item[R4] Meek rule R4 cannot trigger to orient $u \to v$ since it implies that there is a $c \to d \to v$ but $c \to v$ is \emph{not} in the graph. This cannot happen by construction.
    \item[R2] Suppose, for a contradiction, that Meek rule R2 triggers.
    This implies that the arcs $u \to z'$ and $z' \to v$ are oriented for some $z' \in \Des(u) \cap \Anc(v)$.
    If $\Des(u) \cap \Anc(v) = \emptyset$, then this case cannot happen.
    Otherwise, let $z$ be the \emph{earliest} such vertex (i.e.\ $z \in \Anc[z']$ for any such $z'$) and consider the edge $u \to z$.
    By choice of $z$, Meek rule R2 did not orient $u \to z$ when we intervene on $w \in \Anc(u)$.
    From the other case analyses, we see that $u \to z$ is also not oriented by the other Meek rules.
    Therefore, contradicting the implication that $u \to v$ was oriented via Meek rule R2 due to $u \to z$ being oriented.
    \qedhere
\end{description}
\end{proof}

\subsection{Subset verification with bounded size interventions and additive vertex costs}

In this section, we follow the proof strategy of \cite{choo2022verification}, generalizing their results for $T=E$ to arbitrary subset of target edges $T \subseteq E$.
One crucial difference in our approaches is that they rely on the bipartiteness of the covered edges of $G^*$ while we rely on \cref{lem:there-is-a-tree-subset} to argue that there is a way to 2-color the atomic minimum subset verifying set.

\thereisatreesubset*
\begin{proof}
If $G[S]$ is a forest, the claim trivially holds.
Otherwise, we apply the following recursive argument to tranform $S$: as long as $G$ still contains an undirected cycle, we can update $S$ to $S'$ such that $G[S']$ has fewer cycles than $G[S]$ while still ensuring that $R(G,S) \subseteq R(G,S')$.

Let $\pi$ be an arbitrary valid ordering for $G$.
Suppose $G[S]$ contains an undirected cycle $C = r \to u_1 \to \ldots \to u_k \to s \gets v_{\ell} \gets \ldots \gets v_1 \gets r$ of length $|C| = k + \ell + 2 \geq 3$, where $r = u_0 = v_0 = \argmin_{z \in V(C)} \{\pi(z)\}$ and $s = \argmax_{z \in V(C)} \{\pi(z)\}$.
We write $C = r \to s \gets v_{\ell} \gets \ldots \gets v_1 \gets r$ and $C = r \to u_1 \to \ldots \to u_k \to s \gets r$ if $k = 0$ or $\ell = 0$ respectively.

Since $G$ has no v-structures, we must have $v_l \sim u_k$ in $G$.
Without loss of generality, suppose $v_{\ell} \to u_k$.
Then, we update $S$ to $S' = S \cup \{v_{\ell} \to u_k\} \setminus \{v_{\ell} \to s\}$.
Note that $v_{\ell}, s \in S$, so the vertices of the endpoints in $S'$ are a subset of $S$.
Observe that $R(G,S) \subseteq R(G,S')$ because Meek rule R2 will orient $v_{\ell} \to s$ via $v_{\ell} \to u_k \to s \sim v_{\ell}$.
Furthermore, the cycle $C$ is either destroyed (if $|C| = 3$) or is shortened by one (if $|C| > 3$).
We can repeat this edge replacement argument until $G[S']$ has strictly one less undirected cycle than $G[S]$, and eventually until $G[S']$ has no undirected cycles, i.e.\ $G[S']$ is a forest.

It remains to argue that the recursive procedure described above runs in polynomial time.
We first note that cycle finding can be done in polynomial time using depth-first search (DFS).
Now, consider the potential function $\phi(S) = \sum_{e = (u,v) \in S} \pi(u) + \pi(v)$.
In each round, $\phi(S)$ decreases since we replace $v_{\ell} \to u_k$ by $v_{\ell} \to s$ and $\pi(u_k) < \pi(s)$.
Since the initial potential function value is polynomial in $n$, and we decrease it by at least 1 in each step, the entire procedure runs in polynomial time.
\end{proof}

\begin{definition}[Separation of covered edges \cite{choo2022verification}]
We say that an intervention $S \subseteq V$ \emph{separates} a covered edge $u \sim v$ if $|\{u,v\} \cap S| = 1$.
That is, \emph{exactly} one of the endpoints is intervened by $S$.
We say that an intervention set $\cI$ separates a covered edge $u \sim v$ if there exists $S \in \cI$ that separates $u \sim v$.
\end{definition}

\begin{lemma}[Lemma 29 of \cite{choo2022verification}]
\label{lem:atomic-only-helps}
Fix an essential graph $\cE(G^*)$ and $G \in [G^*]$.
Suppose $\cI$ is an arbitrary bounded size intervention set.
Intervening on vertices in $\cup_{S \in \cI} S$ one at a time, in an atomic fashion, can only increase the number of separated covered edges of $G$.
\end{lemma}

\begin{lemma}
\label{lem:bounded-size-lb}
Fix an essential graph $\cE(G^*)$ and $G \in [G^*]$.
If $\nu_1(G,T) = \ell$, then $\nu_k(G,T) \geq \lceil \frac{\ell}{k} \rceil$.
\end{lemma}
\begin{proof}
A bounded size intervention set of size strictly less than $\lceil \frac{\ell}{k} \rceil$ involves strictly less than $\ell$ vertices.
By \cref{lem:atomic-only-helps}, intervening on the vertices of the bounded size intervention set one at a time (i.e.\ simulate it as an atomic intervention set) can only increase the number of oriented edges.
However, such an atomic intervention set cannot be a subset verifying set since it involves strictly less than $\ell$ vertices because $\nu_1(G,T) = \ell$.
\end{proof}

\begin{lemma}
\label{lem:bounded-size-ub}
Fix an essential graph $\cE(G^*)$ and $G \in [G^*]$.
If $\nu_1(G,T) = \ell$, then there exists a polynomial time algorithm that computes a bounded size subset verifying set $\cI$ of size $|\cI| \leq \lceil \frac{\ell}{k} \rceil + 1$.
\end{lemma}
\begin{proof}
Consider any atomic subset verifying set $\cI$ of $G$ of size $\ell$.
Let $S$ be the set of edges incident to vertices in $\cI$.
By \cref{lem:there-is-a-tree-subset}, there is a subset $S' \subseteq E$ such that $G[S']$ is a forest, $R(G,\cI) = R(G,S) \subseteq R(G,S')$ and $\bigcup_{(u,v) \in S'} \{u,v\} \subseteq \bigcup_{(u,v) \in S} \{u,v\}$.
Since $G[S']$ is a forest and $V(G[S']) \subseteq \cI$, there is a 2-coloring of the vertices in $\cI$.

Split the vertices in $\cI$ into partitions according to the 2-coloring.
By construction, vertices belonging in the same partite will \emph{not} be adjacent and thus choosing them together to be in an intervention $S$ will \emph{not} reduce the number of separated covered edges.
Now, form interventions of size $k$ by greedily picking vertices in $\cI$ within the same partite.
For the remaining unpicked vertices (strictly less than $k$ of them), we form a new intervention with them.
Repeat the same process for the other partite.

This greedy process forms groups of size $k$ and at most 2 groups of sizes, one from each partite.
Suppose that we formed $z$ groups of size $k$ in total and two ``leftover groups'' of sizes $x$ and $y$, where $0 \leq x,y < k$.
Then, $\ell = z \cdot k + x + y$, $\frac{\ell}{k} = z + \frac{x+y}{k}$, and we formed at most $z + 2$ groups.
If $0 \leq x+y < k$, then $\lceil \frac{\ell}{k} \rceil = z+1$.
Otherwise, if $k \leq x+y < 2k$, then $\lceil \frac{\ell}{k} \rceil = z+2$.
In either case, we use at most $\lceil \frac{\ell}{k} \rceil + 1$ interventions, each of size $\leq k$.

One can compute a bounded size intervention set efficiently because the following procedures can all be run in polynomial time:
(i) \cref{lem:there-is-a-tree-subset} runs in polynomial time;
(ii) 2-coloring a tree;
(iii) greedily grouping vertices into sizes $\leq k$.
\end{proof}

\boundedsizeextension*
\begin{proof}
Follows by combining \cref{lem:bounded-size-lb} and \cref{lem:bounded-size-ub}.
\end{proof}

To solve the subset verification problem with respect to \cref{eq:generalized-cost}, we need to compute a \emph{weighted} minimum interval stabbing set on a rooted tree.

\begin{lemma}
Fix an essential graph $\cE(G^*)$ and $G \in [G^*]$.
An atomic subset verifying set for $G$ that minimizes \cref{eq:generalized-cost} can be computed in polynomial time.
\end{lemma}
\begin{proof}
Replace $\alpha_v = 1 + \sum_{y \in \Ch(v)} \texttt{DP}(y, \max\{a_{y}, i\})$ by $\alpha_v = w(v) + \sum_{y \in \Ch(v)} \texttt{DP}(y, \max\{a_{y}, i\})$ in \cref{alg:dp}.
\end{proof}

\begin{lemma}
\label{lem:modified-atomic}
Fix an essential graph $\cE(G^*)$ and $G \in [G^*]$.
Let $\cI_A$ be an atomic subset verifying set for $G$ that minimizes $\alpha \cdot w(\cI_A) + \frac{\beta}{k} \cdot |\cI_A|$ and $\cI_B$ be a bounded size verifying set for $G$ that minimizes \cref{eq:generalized-cost}.
Then, $\alpha \cdot w(\cI_A) + \frac{\beta}{k} \cdot |\cI_A| \leq \alpha \cdot w(\cI_B) + \beta \cdot |\cI_B|$.
\end{lemma}
\begin{proof}
Let $\cI = \sum_{S \in \cI_B} S$ be the atomic subset verifying set derived from $\cI_B$ by treating each vertex as an atomic intervention.
Clearly, $w(\cI_B) \geq w(\cI)$ and $k \cdot |\cI_B| \geq |\cI|$.
So,
\[
\alpha \cdot w(\cI_B) + \beta \cdot |\cI_B|
\geq \alpha \cdot w(\cI) + \frac{\beta}{k} \cdot |\cI|
\geq \alpha \cdot w(\cI_A) + \frac{\beta}{k} \cdot |\cI_A|
\]
since $\cI_A = \argmin_{\text{atomic subset verifying set } \cI'} \left\{ \alpha \cdot w(\cI') + \frac{\beta}{k} \cdot |\cI'| \right\}$.
\end{proof}

\additivecostextension*
\begin{proof}
Let $\cI_A$ be an atomic subset verifying set for $G$ that minimizes $\alpha \cdot w(\cI_A) + \frac{\beta}{k} \cdot |\cI_A|$ and $\cI_B$ be a bounded size verifying set for $G$ that minimizes \cref{eq:generalized-cost}.
Using the polynomial time greedy algorithm in \cref{lem:bounded-size-ub}, we construct bounded size intervention set $\cI$ by greedily grouping together atomic interventions from $\cI_A$.
Clearly, $w(\cI) = w(\cI_A)$ and $|\cI| \leq \lceil \frac{|\cI_A|}{k} \rceil + 1$.
So,
\[
\alpha \cdot w(\cI) + \beta \cdot |\cI|
\leq \alpha \cdot w(\cI_A) + \beta \cdot \left( \left\lceil \frac{|\cI_A|}{k} \right\rceil + 1 \right)
\leq \alpha \cdot w(\cI_B) + \beta \cdot |\cI_B| + 2 \beta
= OPT + 2 \beta
\]
where the second inequality is due to \cref{lem:modified-atomic}.
\end{proof}

\subsection{Subset search}
\label{sec:appendix-subsetsearch}

\vertexcoveristrivialupperbound*
\begin{proof}
Each intervention will orient all the incident edges.
\end{proof}

\vertexcoverisworstcasenecessary*
\begin{proof}
Let $[2n]$ be the vertex set. We construct our lower bound graph as follows (see \cref{fig:adaptive-lower-bound} for an illustration):
\begin{itemize}
    \item A clique on first $1$ to $n$ vertices, where their relative orderings can be chosen by the adaptive adversary.
    \item Add an edge between vertex $i \in [n]$ to $n+i$ and restrict that the vertices outside the clique come after the clique nodes in any order.
        Then, the essential graph has no v-structures.
    \item Let $T=\{(i,n+i) \mid i \in [n] \}$ be the set of target edges.
        Note that its minimum vertex cover has size $\omega(n)$.
\end{itemize}

To orient all the edges in set $T$, we just need to orient on the source vertex of the clique and then apply Meek rules.
Therefore, $\nu_{1}(G^*,T)=1$ for any graph $G^*$ in this equivalence class.

Meanwhile, figuring out the source vertex $s$ and $n+s$ for any search algorithm will require at least $n-1$ for any adaptive search algorithm when we have an adapative adversary.
Since intervening on vertices outside the clique only learns the incident arc itself while intervening on the other endpoint in the clique recovers more arc orientations, we may assume without loss of generality that search algorithms will only intervene on vertices within the clique.
Now, to orient all the edges in the set $T$, we need to figure out the source vertex and it is well known that figuring out the source vertex requires at least $n-1$ queries.
\end{proof}

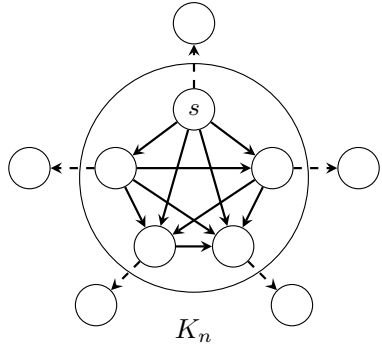
\begin{figure}[htbp]
\centering
\resizebox{0.3\linewidth}{!}{%
\begin{tikzpicture}
\node[draw, circle, minimum size=15pt, inner sep=2pt] at (0,0) (v1) {\small $s$};
\node[draw, circle, minimum size=15pt, inner sep=2pt] at ($(v1)+(-1,-0.75)$) (v2) {};
\node[draw, circle, minimum size=15pt, inner sep=2pt] at ($(v1)+(1,-0.75)$) (v3) {};
\node[draw, circle, minimum size=15pt, inner sep=2pt] at ($(v2)+(0.5,-1)$) (v4) {};
\node[draw, circle, minimum size=15pt, inner sep=2pt] at ($(v3)+(-0.5,-1)$) (v5) {};

\node[draw, circle, minimum size=15pt, inner sep=2pt] at ($(v1)+(0,1.1)$) (v6) {};
\node[draw, circle, minimum size=15pt, inner sep=2pt] at ($(v2)+(-1.1,0)$) (v7) {};
\node[draw, circle, minimum size=15pt, inner sep=2pt] at ($(v3)+(1.1,0)$) (v8) {};
\node[draw, circle, minimum size=15pt, inner sep=2pt] at ($(v4)+(-0.75,-0.75)$) (v9) {};
\node[draw, circle, minimum size=15pt, inner sep=2pt] at ($(v5)+(0.75,-0.75)$) (v10) {};

\draw[thick, -stealth] (v1) -- (v2);
\draw[thick, -stealth] (v1) -- (v3);
\draw[thick, -stealth] (v1) -- (v4);
\draw[thick, -stealth] (v1) -- (v5);
\draw[thick, -stealth] (v2) -- (v3);
\draw[thick, -stealth] (v2) -- (v4);
\draw[thick, -stealth] (v2) -- (v5);
\draw[thick, -stealth] (v3) -- (v4);
\draw[thick, -stealth] (v3) -- (v5);
\draw[thick, -stealth] (v4) -- (v5);
\draw[dashed, thick, -stealth] (v1) -- (v6);
\draw[dashed, thick, -stealth] (v2) -- (v7);
\draw[dashed, thick, -stealth] (v3) -- (v8);
\draw[dashed, thick, -stealth] (v4) -- (v9);
\draw[dashed, thick, -stealth] (v5) -- (v10);

\node[draw, fit=(v1)(v2)(v3)(v4)(v5), circle, inner sep=-5pt] {};
\node[text centered, below=65pt of v1] {$K_n$};
\end{tikzpicture}
}
\caption{
Adaptive lower bound construction of $G$ with $n = 5$: Given an integer $n \geq 1$, construct a directed clique $K_{n}$ and have each clique node point to a fresh node outside of the clique.
The dashed $n$ dashed arcs are chosen to be the target edges $T \subseteq E$.
The essential graph $\cE(G)$ is completely undirected and any permutation ordering on the clique nodes are valid.
Intervening on the source $s$ of the clique is sufficient to fully orient $T$ with the aid of Meek rules.
However, an adaptive adversary can always decide that vertices outside the clique have come after the the clique nodes in the ordering, and always decide that the $i^{th}$ vertex $v$ in the clique that we intervene on within the clique has ordering $\pi(v) = n-i+1$, and thus we only learn the orientations of arcs incident to $v$.
}
\label{fig:adaptive-lower-bound}
\end{figure}

To prove \cref{thm:subset-search-on-node-induced-subgraph}, we modify \cite[Algorithm 1]{choo2022verification} by only assigning non-zero weights on vertices from $V(H)$.

\begin{algorithm}[htbp]
\caption{\texttt{SubsetSearch}: Node-induced subset search algorithm via weighted graph separators.}
\label{alg:subset-search-algo}
\begin{algorithmic}[1]
    \State \textbf{Input}: Interventional essential graph $\cE_{\cI}(G^*)$, node-induced subgraph $H$, intervention upper bound size $k \geq 1$.
    \State \textbf{Output}: A partially oriented interventional essential graph $G$ such that $G[V(H)] = G^*[V(H)]$.
    \State Initialize $i=0$ and $\cI_0 = \emptyset$.
	\While{$\cE_{\cI_{i}}(G^*)[V(H)]$ still has undirected edges}
    	\State For each $H_{CC} \in CC(\cE_{\cI_{i}}(G^*))$ with $|\rho(\cI_{i} \cup \cI, V(H_{CC}))| \geq 2$, find a 1/2-clique separator $K_H$ using \Statex\hspace{\algorithmicindent}\cref{thm:chordal-separator-weighted},
    	with respect to weight function
    	\[
    	c(v) =
    	\begin{cases}
    	\frac{n}{\rho(\cI_i \cup \cI, V(H_{CC}))} & \text{for $v \in V(H)$}\\
    	0 & \text{for $v \in V \setminus V(H)$}
    	\end{cases}
    	\]
    	\State Define $Q \subseteq V$ as the union of clique separator nodes across all $H_{CC} \in CC(\cE_{\cI_{i}}(G^*))$.
    	\If{$k=1$ or $|Q|=1$}
    	    \State Define $C_{i+1} = Q$ as an atomic intervention set.
    	\Else
    	    \State Define $k' = \min\{k, |Q|/2\}$, $a = \lceil |Q|/k' \rceil \geq 2$, and $\ell = \lceil \log_a n \rceil$. Compute labelling scheme of
    	\Statex\hspace{\algorithmicindent}\hspace{\algorithmicindent}\cite[Lemma 1]{shanmugam2015learning} on $Q$ with $(|Q|, k', a)$, and define $C_{i+1} = \{S_{x,y}\}_{x \in [\ell], y \in [a]}$, where $S_{x,y} \subseteq Q$ is \Statex\hspace{\algorithmicindent}\hspace{\algorithmicindent}the subset of vertices whose $x^{th}$ letter in the label is $y$.
    	\EndIf
    	\State Update $i \gets i+1$, intervene on $C_{i}$ to obtain $\cE_{\cI_{i}}(G^*)$, and update $\cI_i \gets \cI_{i-1} \cup C_{i}$.
	\EndWhile
	\State \Return $\cI_i$ \Comment{Note: The algorithm does \emph{not} know $\cI$. We write $\cI_i \cup \cI$ for notational purposes only.}
\end{algorithmic}
\end{algorithm}

For analysis, we rely on the following known results \cref{lem:strengthened-lb} and \cref{thm:chordal-separator-weighted}.

\begin{lemma}[Lemma 21 of \cite{choo2022verification}]
\label{lem:strengthened-lb}
Fix an essential graph $\cE(G^*)$ and $G \in [G^*]$.
Then,
\[
\nu_1(G, E) \geq \max_{\cI \subseteq V} \sum_{H \in CC(\cE_{\cI}(G^*))} \left\lfloor \frac{\omega(H)}{2} \right\rfloor
\]
\end{lemma}

\begin{lemma}[\cite{gilbert1984separatorchordal}]
\label{thm:chordal-separator-weighted}
Let $G = (V,E)$ be a chordal graph with $|V| \geq 2$ and $p$ vertices in its largest clique.
Suppose each vertex $v$ is assigned a non-negative weight $c(v) \geq 0$ such that $\sum_{v} c(v) = n$.
Then, there exists a $1/2$-clique-separator $C$ of size $|C| \leq p - 1$ such that any connected component in $G$ after the removal has total weight of no more than $\sum_{v \in V} c(v)/2$.
The clique $C$ can be computed in $\cO(|E|)$ time.
\end{lemma}

Our analysis approach mirrors \cite{choo2022verification}: we first argue that \cref{alg:subset-search-algo} terminates after $\cO(\log |V(H)|)$ iterations, and then argue that each iteration uses at most $\cO(\nu_1(G^*))$ atomic interventions.

\begin{lemma}
\label{lem:log-rounds-suffice}
\cref{alg:subset-search-algo} terminates after at most $\cO(\log |\rho(\cI, V(H))|)$ iterations.
\end{lemma}
\begin{proof}
Note that chain components in $\cE_{\cI_{i}}(G^*)[V(H)]$ are chordal since node-induced subgraphs of a chordal graph are also chordal.
By choice of $c(v)$ and \cref{thm:chordal-separator-weighted}, all connected components will have total weight.
Since each vertex in $V(H_{CC}) \cap V(H)$ is assigned the same weight via $c(v)$, the number of vertices from $V(H)$ within any connected component $H_{CC}$ is at least halved per iteration.
Thus, after $\cO(\log |V(H)|)$ iterations, all connected components have at most one vertex from $V(H)$, which in turn means that all edges within the node-induced graph $H$ has been oriented.
\end{proof}

\subsetsearchonnodeinducedsubgraph*
\begin{proof}
Fix an arbitrary iteration $i$ of \cref{alg:subset-search-algo} and let $G_i$ be the partially oriented graph obtained after intervening on $\cI_i$.
By \cref{lem:strengthened-lb}, $\sum_{H \in CC(\cE_{\cI_i}(G^*))} \lfloor \frac{\omega(H)}{2} \rfloor \leq \nu_1(G^*, E)$.
By definition of $\omega$, we always have $|K_H| \leq \omega(H)$.
Thus, \cref{alg:subset-search-algo} uses at most $2 \cdot \nu_1(G^*, E)$ interventions in each iteration.

By \cref{lem:log-rounds-suffice}, \cref{alg:subset-search-algo} terminates after $\cO(\log |\rho(\cI, V(H))|)$ iterations and so the algorithm uses at most $\cO(\log (|\rho(\cI, V(H))|) \cdot \nu_1(G^*, E))$ atomic interventions in total.
\end{proof}

\subsetsearchonnodeinducedsubgraphbounded*
\begin{proof}
Fix an arbitrary iteration $i$ of \cref{alg:subset-search-algo} and let $G_i$ be the partially oriented graph obtained after intervening on $\cI_i$.
Applying exactly the same proof as \cite[Theorem 16]{choo2022verification}, we see that $|C_i| \in \cO\left(\nu_k(G^*, E) \cdot \log k \right)$.
By \cref{lem:log-rounds-suffice}, there are $\cO(\log |\rho(\cI, V(H))|)$ iterations and so $\cO(\log(|\rho(\cI, V(H))|) \cdot \log (k) \cdot \nu_k(G^*, E))$ bounded size interventions are used by \cref{alg:subset-search-algo}.
\end{proof}

\subsection{Interval stabbing problem on a rooted tree}

\recurrencecorrectness*
\begin{proof}
Consider any arbitrary vertex $v \in V$ and any child $y \in \Ch(v)$ such that $C_v \cap I_y \neq \emptyset$.
For any $U \subseteq \cJ$, define
\[
L_{U,v,y} = \argmin_{[a,b] \in U \cap C_v \cap I_y} \{ \pi(b) \} = [a_{v,y}, b_{v,y}]
\]
as the earliest ending interval within $U$ that is covered by $v$ in subtree $T_y$.

Suppose $v \not\in \cI$ and $\cI$ does not include any ancestor of $v$.
To stab any interval in $[a,b] \in C_v$, we must have $w \in \cI$ for some $w \in \Des(v) \cap \Anc[b]$.
Since subtrees $T_y$ are disjoint, we can partition $C_v$ into $\;\dot\cup\;_{y \in \Ch(v)} C_{v,y} = \;\dot\cup\;_{y \in \Ch(v)} C_v \cap I_y$, where $C_{v,y}$ is associated to subtree $T_y$.
So, for each interval $[a,b] \in C_{v,y}$, we need to ensure that $w \in \cI$ for some $w \in \Des(v) \cap \Anc[b]$.
By minimality of $b_{v,y}$, stabbing $L_{U,v,y}$ ensures that all intervals in $C_{v,y}$ are stabbed and any stabbing for $C_{v,y}$ must also stab $L_{U,v,y}$.
\end{proof}

\section{Efficient dynamic programming implementation of recurrence}
\label{sec:appendix-dp}

We first recall the definitions and recurrence equations established in \cref{sec:interval-stabbing-on-a-tree} before explaining how to solve \cref{defn:IC} in polynomial time via dynamic programming (DP).

\subsection{Recap}

Given a set of intervals $\cJ$, we define the following sets with respect to an arbitrary vertex $v \in V$:
\begin{align*}
    E_v &= \{[a,b]_{G} \in \cJ : b=v\} && \text{(End with $v$)}\\
    M_v &= \{[a,b]_{G} \in \cJ : v \in (a,b)_{G}\} && \text{(Middle with $v$)}\\
    S_v &= \{[a,b]_{G} \in \cJ : a=v\} && \text{(Start with $v$)}\\
    W_v &= \{[a,b]_{G} \in \cJ : a,b \in V(T_{v}) \setminus \{v\}\} && \text{(Without $v$)}\\
    I_v &= E_v \cup M_v \cup S_v \cup W_v && \text{(Intersect $T_v$)}\\
    B_v &= S_v \cup W_v && \text{(Back of $I_v$)}\\
    C_v &= E_v \cup M_v \cup S_v && \text{(Covered by $v$)}
\end{align*}
Note that $I_v$ includes all the intervals in $\cJ$ that intersect with the subtree $T_v$ (i.e.\ has some vertex in $V(T_v)$) and $C_v$ includes all the intervals that will be covered whenever $v \in \cI$.
Observe that $I_y \subseteq I_v$ for any $y \in \Des(v)$.

For any subset $U \subseteq \cJ$ and vertex $v \in V$, $\opt(U,v)$ denotes the \emph{size} of the optimum solution to stab all the intervals in $U$ using only vertices in $V(T_v)$ in the subtree $T_v$ rooted at $v$, where
\begin{equation*}
\opt(U, v) =
\begin{cases}
\infty & \text{if $U \not\subseteq I_v$}\\
\alpha_v & \text{if $U \subseteq I_v$, $U \cap E_v \neq \emptyset$}\\
\min\{\alpha_v, \beta_v\} & \text{if $U \subseteq I_v$, $U \cap E_v = \emptyset$}
\end{cases}
\end{equation*}
\[
\text{where}\qquad
\alpha_v = 1 + \sum_{y \in \Ch(v)} \opt(U \cap B_y, y)
\qquad\text{and}\qquad
\beta_v = \sum_{y \in \Ch(v)} \opt(U \cap I_y, y)
\]

That is, we must pick $v \in \cI$ whenever $E_v \neq \emptyset$, while $\alpha_v$ and $\beta_v$ correspond to the decisions of picking $v$ into the output and ignoring $v$ from the output respectively.
Then, $\opt(\cJ, r)$ is the optimum solution size to the interval stabbing problem, where $r$ as the root of the given rooted tree.

\subsection{Efficient implementation}

Naively computing the recurrence relation of \cref{eq:dp-recurrence} will incur an exponential blow-up in state space.
Instead, we will define an ordering $\prec$ on $\cJ$ so that our state space is over the indices of a sorted array instead of a subset of intervals (see \cref{eq:prec-ordering}), so that we can implement the recurrence as a polynomial time dynamic programming (DP) problem.

Our $\prec$ ordering relies on the Euler tree data structure for rooted trees \cite{tarjan1984finding,henzinger1995randomized}, which computes a sequence $\tau$ of vertices visited in a depth-first search (DFS) from the root.
Using this sequence $\tau$, we can obtain the first ($f$) and last ($\ell$) times that a vertex is visited.
More formally, we can define the mappings $\tau : \{1, \ldots, 2n-1\} \to V$, $f: V \to \{1, \ldots, 2n-1\}$, and $\ell: V \to \{1, \ldots, 2n-1\}$.
These mappings can be computed in linear time (via DFS) and $f(v) \leq \ell(v)$ with equality only if $v$ is a leaf of the tree.
See \cref{fig:tree-example} for an illustration of $\tau$, $f$, and $\ell$.

\begin{restatable}[Properties of Euler tour data structure]{lemma}{eulertourproperties}
\label{euler-tour-properties}
Given an Euler tour data structure sequence $\tau$ with corresponding $f$ and $\ell$ indices, the following statements hold:
\begin{enumerate}
    \item $V(T_v) = \cup_{f(v) \leq i \leq \ell(v)} \tau(i)$.
    \item If $f(v) < f(u) < \ell(v)$, then $f(v) < \ell(u) < \ell(v)$.
    \item If $f(v) < \ell(u) < \ell(v)$, then $f(v) < f(u) < \ell(v)$.
    \item If $f(v) < f(u) < \ell(v)$ \emph{or} $f(v) < \ell(u) < \ell(v)$, then $f(v) < f(u), \ell(u) < \ell(v)$.
        That is, $u \in V(T_v)$.
    \item For any interval $[a,b]_G$, $f(a) < f(b) \leq \ell(b) < \ell(a)$.
    \item For any interval $[a,b]_G$ and any vertex $z \in V$, $[a,b]_G \in I_z \iff b \in T_z$.
\end{enumerate}
\end{restatable}
\begin{proof}
\begin{enumerate}
    \item By definition of DFS traversal from the root.
    \item Suppose, for a contradiction, that $f(v) < f(u) < \ell(v) < \ell(u)$.
        This cannot happen in a DFS traversal on a tree.
    \item Suppose, for a contradiction, that $f(u) < f(v) < \ell(u) < \ell(v)$.
        This cannot happen in a DFS traversal on a tree.
    \item Combine above properties.
    \item $a \in \Anc[b]$ in the tree $G$ by definition of intervals.
    \item By definition of $I_z$ and $T_z$.\qedhere
\end{enumerate}
\end{proof}

Using the Euler tour data structure, we can efficiently remove a subset of ``unnecessary intervals'' from $\cJ$, whose removal will not affect the optimality of the recurrence while granting us some additional structural properties which we will exploit.
We call these ``unnecessary intervals'' \emph{superset intervals}.

\begin{definition}[Superset interval]
We say that an interval $[c,d] \in \cJ$ is a \emph{superset interval} if there exists another interval $[a,b] \in \cJ$ such that $c \in \Anc[a]$ and $b \in \Anc[d]$.
Note that $a \in \Anc[b]$ is implied by the fact that $[a,b]$ is an interval.
\end{definition}

Observe that the removal of superset intervals will not affect the optimality of the solution because stabbing $[a,b]$ will stab $[c,d]$.
For an interval $[a,b]$, we call $a$ the \emph{starting vertex} and $b$ the \emph{ending vertex} of the interval $[a,b]$ respectively.
Using the Euler tour data structure, superset intervals can be removed in $\cO(|\cJ| \log |\cJ|)$ time by first sorting the intervals according to the ending vertex, then only keep the intervals with the latest starting vertex amongst any pair of intervals that share the same ending vertex.
After removing superset intervals, we are guaranteed that the ending vertices in $\cJ$ are unique.

We now define an ordering $\prec$ on $\cJ$ using the Euler tour mapping $f$ so that $\cJ[i] \prec \cJ[i]$ for any $i < j$:
\begin{equation}
\label{eq:prec-ordering}
[a,b] \prec [c,d] \iff f(a) < f(c) \text{ or } ( a = c \text{ and } \ell(b) > \ell(d) )
\end{equation}
We write $\cJ^{-1}([a,b])$ to refer to the index of $[a,b]$ in $\cJ$.
Since $I_y \subseteq I_v$ for any $y \in \Ch(v)$, we are guaranteed that $\min_{[a,b] \in I_v} \cJ^{-1}([a,b]) \leq \min_{[a,b] \in I_y} \cJ^{-1}([a,b])$ for any $y \in \Ch(v)$.
However, note that there may be intervals \emph{outside} of $I_v$ with indices between $\min_{[a,b] \in I_v} \cJ^{-1}([a,b])$ and $\max_{[a,b] \in I_v} \cJ^{-1}([a,b])$.
For any vertex $v \in V$, we define
\[
\cJ_v
= sorted \left( \{[a,b] \in \cJ \cap I_v: \cJ^{-1}([a,b])\} \right)
\]
as the array of indices of intervals in $I_v$ such that $\cJ_v[i] \prec \cJ_v[j]$ for all $1 \leq i < j \leq |I_v| = |\cJ_v|$.
See \cref{fig:J-example}.

\begin{figure}[htbp]
\centering
\resizebox{0.5\linewidth}{!}{%
\begin{tikzpicture}
\node[draw, circle, minimum size=15pt, inner sep=2pt] at (0,0) (a) {\small $a$};
\node[draw, circle, minimum size=15pt, inner sep=2pt, below=20pt of a] (c) {\small $c$};
\node[draw, circle, minimum size=15pt, inner sep=2pt, left=25pt of c] (b) {\small $b$};
\node[draw, circle, minimum size=15pt, inner sep=2pt, right=25pt of c] (d) {\small
$d$};
\node[draw, circle, minimum size=15pt, inner sep=2pt, below=20pt of b, xshift=-20pt] (e) {\small $e$};
\node[draw, circle, minimum size=15pt, inner sep=2pt, below=20pt of b, xshift=20pt] (f) {\small $f$};
\node[draw, circle, minimum size=15pt, inner sep=2pt, below=20pt of c] (g) {\small $g$};
\node[draw, circle, minimum size=15pt, inner sep=2pt, below=20pt of d] (h) {\small $h$};
\node[draw, circle, minimum size=15pt, inner sep=2pt, below=20pt of h, xshift=-20pt] (i) {\small $i$};
\node[draw, circle, minimum size=15pt, inner sep=2pt, below=20pt of h, xshift=20pt] (j) {\small $j$};

\draw[thick, -stealth] (a) -- (b);
\draw[thick, -stealth] (a) -- (c);
\draw[thick, -stealth] (a) -- (d);
\draw[thick, -stealth] (b) -- (e);
\draw[thick, -stealth] (b) -- (f);
\draw[thick, -stealth] (c) -- (g);
\draw[thick, -stealth] (d) -- (h);
\draw[thick, -stealth] (h) -- (i);
\draw[thick, -stealth] (h) -- (j);

\draw[thick, dashed, Bracket-Bracket] ($(a) + (-0.1,0.5)$) to[bend right=20] ($(b) + (-0.5,0)$);
\draw[thick, dashed, Bracket-Bracket] ($(a) + (-0.1,0.75)$) to[bend right=30] ($(e) + (-0.5,0)$);
\draw[thick, dashed, Bracket-Bracket] ($(c) + (0.5,0)$) -- ($(g) + (0.5,0)$);
\draw[thick, dashed, Bracket-Bracket] ($(a) + (0.1,0.5)$) to[bend left=50] ($(h) + (0.75,0)$);
\draw[thick, dashed, Bracket-Bracket] ($(d) + (0.5,0)$) to[bend right=35] ($(j) + (0.5,0)$);
\draw[thick, dashed, Bracket-Bracket] ($(a) + (0.25,-0.5)$) -- ($(d) + (-0.35,0)$) -- ($(h) + (-0.35,0)$) -- ($(i) + (-0.5,0)$);
\end{tikzpicture}
}
\caption{
Consider the rooted tree $G$ on $n=10$ vertices with intervals $\cJ = \{[a,b], [a,e], [a,h], [a,i], [c,g], [d,j]\}$.
Here, $[a,e]$ is a superset interval of $[a,b]$.
We see that $\opt(\cJ, a) = 3$, where $\{a,c,d\}$ and $\{b,g,h\}$ are possible optimal sized interval covers.
One possible Euler tour sequence $\tau$ is $(a,b,e,b,f,b,a,c,g,c,a,d,h,i,h,j,h,d,a)$ of length $|\tau| = 2n-1 = 19$.
\cref{tab:tree-example} shows the first ($f$) and last ($\ell$) indices within $\tau$.
Observe that the leaves $e,f,g,i,j$ have the same first and last indices, and vertices $d,h,i,j \in V(T_d)$ have indices between $f(d) = 12$ and $\ell(d) = 18$.
Under \cref{eq:prec-ordering}, we have $[a,h] \prec [a,i] \prec [a,b] \prec [a,e] \prec [c,g] \prec [d,j]$.
}
\label{fig:tree-example}
\end{figure}
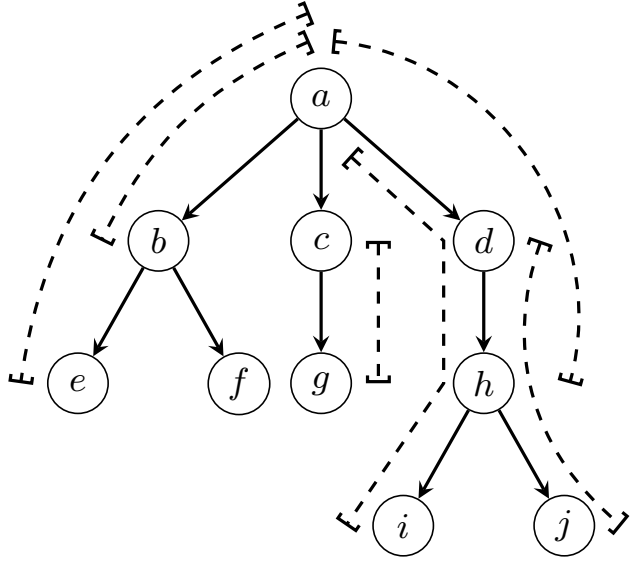
\begin{table}[htbp]
\centering
\begin{tabular}{c|cccccccccc}
\toprule
& $a$ & $b$ & $c$ & $d$ & $e$ & $f$ & $g$ & $h$ & $i$ & $j$\\
\midrule
$f$ & 1 & 2 & 8 & 12 & 3 & 5 & 9 & 13 & 14 & 16\\
$\ell$ & 19 & 6 & 10 & 18 & 3 & 5 & 9 & 17 & 14 & 16\\
\bottomrule
\end{tabular}
\caption{First ($f$) and last $(\ell)$ indices for \cref{fig:tree-example}}
\label{tab:tree-example}
\end{table}

\begin{figure}[htbp]
\centering
\resizebox{0.3\linewidth}{!}{%
\begin{tikzpicture}
\node[draw, circle, minimum size=15pt, inner sep=2pt] at (0,0) (a) {\small $a$};
\node[draw, circle, minimum size=15pt, inner sep=2pt, below=20pt of a] (b) {\small $b$};
\node[draw, circle, minimum size=15pt, inner sep=2pt, below=20pt of b, xshift=-20pt] (c) {\small $c$};
\node[draw, circle, minimum size=15pt, inner sep=2pt, below=20pt of b, xshift=20pt] (d) {\small $d$};
\node[draw, circle, minimum size=15pt, inner sep=2pt, below=20pt of d] (e) {\small $e$};

\draw[thick, -stealth] (a) -- (b);
\draw[thick, -stealth] (b) -- (c);
\draw[thick, -stealth] (b) -- (d);
\draw[thick, -stealth] (d) -- (e);

\draw[thick, dashed, Bracket-Bracket] ($(b) + (-0.25,0.5)$) -- ($(c) + (-0.5,0)$);
\draw[thick, dashed, Bracket-Bracket] ($(a) + (0.5,0)$) to[bend right=30] ($(d) + (0.5,0)$);
\draw[thick, dashed, Bracket-Bracket] ($(d) + (-0.5,0)$) -- ($(e) + (-0.5,0)$);
\end{tikzpicture}
}
\caption{
Consider the rooted tree $G$ on 5 vertices with intervals $\cJ = \{[a,d], [b,c], [d,e]\}$ and the Euler tour visits $a,b,c,d,e$ in sequence.
Then, $[a,d] \prec [b,c] \prec [d,e]$ and $\cJ^{-1}([a,d]) = 1$, $\cJ^{-1}([b,c]) = 2$, and $\cJ^{-1}([d,e]) = 3$.
In this example, $I_d = \{[a,d], [d,e]\}$ and $\cJ_d = [1,3]$.
Observe that $\min_i \cJ_v[i] < \cJ^{-1}([b,c]) < \max_i \cJ_v[i]$ despite $[b,c] \not\in I_d$.
}
\label{fig:J-example}
\end{figure}
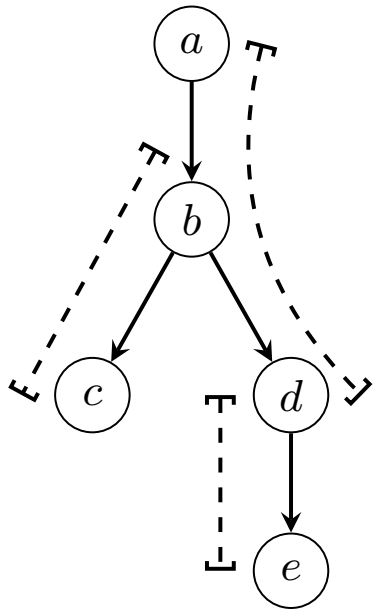

We begin with a simple lemma relating the first time a depth-first search visits a vertex and the ancestry of vertices.

\begin{lemma}
\label{lem:dfs-ancestry}
Consider arbitrary vertices $a,b,v \in V$ in a rooted tree $G$ with root $r$.
If $a,b \in \Anc(v)$, then either $a \in \Anc[b]$ or $b \in \Anc[a]$.
Furthermore, if $f(a) \leq f(b)$ then $a \in \Anc[b]$.
\end{lemma}
\begin{proof}
Since $G$ is a rooted tree, there is a unique path $P$ from $r$ to any vertex $v \in V$ that involves all ancestors of $v$.
Since $a,b \in \Anc(v)$, then either $a$ appears before $b$ in $P$ (i.e.\ $a \in \Anc[b]$) or $b$ appears before $a$ in $P$ (i.e.\ $b \in \Anc[a]$).
By definition of depth-first search from $r$, if we visit $a$ before $b$ (i.e.\ $f(a) \leq f(b)$), then it must be the case that $a \in \Anc[b]$.
\end{proof}

The next lemma tells us that unstabbed intervals form a contiguous interval in $\cJ_v$ and that $E_v$ appears first within $\cJ_v$.

\begin{restatable}[Properties of $\cJ$ with respect to $\prec$]{lemma}{propertiesofJ}
\label{lem:properties-of-J}
Consider an arbitrary $v \in V$ where $|I_v| \geq 2$.
\begin{itemize}
    \item For any $1 \leq i < j \leq |I_v|$, if $\cJ_v[j] = [c,d]$ is stabbed by some $z \in \Anc(v)$ then $\cJ_v[i] = [a,b]$ is also stabbed by $z$.
    \item If $E_v \neq \emptyset$ and $I_v \setminus E_v \neq \emptyset$, then $\max_{[a,b] \in E_v} \cJ^{-1}([a,b]) \leq \min_{[a,b] \in I_v \setminus E_v} \cJ^{-1}([a,b])$.
\end{itemize}
\end{restatable}
\begin{proof}
We know that if $[a,b], [c,d] \in I_v$, then $b,d \in T_v$, i.e.\ $b,d \in \Des[v]$, or equivalently, $v \in \Anc[b]$ and $v \in \Anc[d]$.

\textbf{(First property)}
$z$ stabs $[c,d]$, so $z \in \Des[c] \cap \Anc(v)$.
Since $[a,b] \prec [c,d]$, we have $f(a) < f(c)$ or $a = c$.
In either case, $f(a) \leq f(c)$ and so $a \in \Anc[c]$ by \cref{lem:dfs-ancestry}.
So, $z \in \Des[c] \cap \Anc(v) \subseteq \Des[a] \cap \Anc(v) \subseteq \Des[a] \cap \Anc[b]$, i.e.\ $z$ stabs $[a,b]$.

\textbf{(Second property)}
It suffices to argue that $[a,b] \prec [c,d]$ for any $[a,b] \in E_v \subseteq I_v$ and $[c,d] \in I_v \setminus E_v$.

Since $[a,b], [c,d] \in I_v$, we know that $b,d \in T_v$, i.e.\ $v \in \Anc[b]$ and $v \in \Anc[d]$.
Since $[a,b] \in E_v$, we have $b = v$, so $a \in \Anc(b) = \Anc(v)$ and $b = v \in \Anc[d]$.

\textbf{Case 1: $c \not\in \Anc[v]$}\\
Since $[c,d] \in I_v$, we have $c \in \Anc[v]$, i.e.\ $f(v) \leq f(c)$.
Since $a \in \Anc(v)$, this implies that $f(a) < f(c)$.

\textbf{Case 2: $c \in \Anc[v]$}\\
Suppose, for a contradiction, that $f(a) > f(c)$.
Since $a,c \in \Anc[v]$, \cref{lem:dfs-ancestry} tells us that $c \in \Anc(a)$.
Then, $[c,d]$ is a superset interval with respect to $[a,b]$ since $c \in \Anc(a)$, $a \in \Anc(b)$, and $b \in \Anc(d)$.
This is a contradiction since we have removed all superset intervals.

In either case, $f(a) \leq f(c)$, so $[a,b] \prec [c,d]$.
\end{proof}

We now describe our DP algorithm (\cref{alg:mics} and \cref{alg:dp}) where we \emph{always} recurse on subsets within $I_v$ (e.g.\ see line 6 in \cref{alg:dp}).
For any vertex $v \in V$, our DP state will recurse on the smallest index of the remaining unstabbed intervals within $I_v$.
If all intervals within $I_v$ are stabbed, then the recursed index will be $\infty$ and the recursion terminates.

Suppose we are currently recursing on vertex $v$ and index $i$.
Let $U = \{[a,b] \in I_v : J^{-1}([a,b]) \geq i\}$.
To determine whether $U \cap E_v$ is empty, we can define
\[
e_v =
\begin{cases}
\max_{[a,b] \in E_v} \cJ^{-1}([a,b]) & \text{if $E_v \neq \emptyset$}\\
-\infty & \text{if $E_v = \emptyset$}
\end{cases}
\]
and check whether $e_v \geq i$.
This works because \cref{lem:properties-of-J} guarantees that $E_v$ appears in the front of $\cJ_v$, so $e_v \geq i \iff U \cap E_v \neq \emptyset$.
Meanwhile, the appropriate index update for $U \cap B_y$ in the $\alpha_v$ case is $\max\{a_y, i\}$ where
\[
a_y =
\begin{cases}
\min_{[a,b] \in B_y} \cJ^{-1}([a,b]) & \text{if $B_y \neq \emptyset$}\\
\infty & \text{if $B_y = \emptyset$}
\end{cases}
\]
Similarly, the index update $U \cap I_y$ in the $\beta_v$ case is $\max\{b_{v,y}, i\}$ where
\[
b_{v,y} =
\begin{cases}
\min_{[a,b] \in I_y} \cJ^{-1}([a,b]) & \text{if $I_y \neq \emptyset$}\\
\infty & \text{if $I_y = \emptyset$}
\end{cases}
\]
One can verify that all the $e_v, a_{y}, b_{v,y}$ indices can be pre-computed in polynomial time before executing the DP.
To extract a minimum sized stabbing set for $\cJ$ of size $\opt(\cJ, r)$, one can perform a standard backtracing of the memoization table.

\begin{algorithm}[htbp]
\caption{Minimum interval stab size on a rooted tree.}
\label{alg:mics}
\begin{algorithmic}[1]
    \State \textbf{Input}: Rooted tree $G$ with root $r$, set of intervals $\cJ$.
    \State \textbf{Output}: $\opt(\cJ,r) = \texttt{DP}(r, 0)$.
    \State Compute Euler tour mappings $f$ and $\ell$, sort $\cJ$ according to $\prec$ ordering.
    \State Remove superset intervals from $\cJ$.
    \State Pre-compute indices $e_v$, $a_y$ and $b_{v,y}$ for all $v \in V$ and $y \in \Ch(v)$.
    \State \Return $\texttt{DP}(r, 0)$
\end{algorithmic}
\end{algorithm}

\begin{algorithm}[htbp]
\algnewcommand{\IIf}[1]{\State\algorithmicif\ #1\ \algorithmicthen}
\algnewcommand{\EndIIf}{\unskip\ \algorithmicend\ \algorithmicif}
\caption{Dynamic programming subroutine \texttt{DP}.}
\label{alg:dp}
\begin{algorithmic}[1]
    \State \textbf{Input}: Vertex $v$, index $i \in \{0, 1, \ldots, |\cJ| - 1 \}$.
    \State \textbf{Output}: $\texttt{DP}(v, i)$
    \State\algorithmicif\ $i = \infty$ \algorithmicthen\ \Return 0 \Comment{Done processing $\cJ$}
    \State $\alpha_v = 1 + \sum_{y \in \Ch(v)} \texttt{DP}(y, \max\{a_{y}, i\})$
    \State $\beta_v = \sum_{y \in \Ch(v)} \texttt{DP}(y, \max\{b_{v,y}, i\})$
    \State\algorithmicif\ $e_v \geq i$ \algorithmicthen\ $\texttt{memo}(v, i) \leftarrow \alpha_v$\Comment{$U \cap E_v \neq \emptyset$}
    \State\algorithmicelse\ $\texttt{memo}(v, i) \leftarrow \min\{\alpha_v, \beta_v\}$
    \State \Return $\texttt{memo}(v, i)$
\end{algorithmic}
\end{algorithm}

\begin{restatable}{theorem}{polytimeforintervalstabbing}
\label{thm:poly-time-for-interval-stabbing}
Together, \cref{alg:mics} and \cref{alg:dp} correctly output $\opt(\cJ,r)$ in $\cO(n^2 \cdot |\cJ|)$ time.
\end{restatable}
\begin{proof}
\textbf{Correctness}
The indices $e_v, a_y, b_{v,y}$ are defined to match \cref{eq:dp-recurrence} and the correctness follows from \cref{lem:properties-of-J}.

The invariant we maintain throughout the recursion is as follows: $\cJ[i]$ has \emph{not} been stabbed by $\Anc(v)$ whenever we are in a recursive step at some vertex $v \in V$ and index $i$.
We know from \cref{lem:properties-of-J} that any interval $[a,b]$ with $\cJ^{-1}([a,b]) < i$ would have been stabbed.
So, recursing on $\max\{a_y,i\}$ is equivalent to recursing on $U \cap I_y$ and $\max\{b_{v,y},i\}$ is equivalent to recursing on $U \cap B_y$ in \cref{eq:dp-recurrence}, for any $y \in \Ch(v)$.
Since we immediately recurse on the $\alpha_v$ case whenever $U \cap E_v \neq \emptyset$, we avoid the $\infty$ case in \cref{eq:dp-recurrence}.

\textbf{Runtime}
The computation time of Euler tour data structure can be done in $\cO(n)$ time via depth-first-search on the rooted tree.
The removal of superset intervals can be done in $\cO(|\cJ| \log |\cJ|)$ time.
Sorting of $\cJ$ according to the $\prec$ ordering can be done in $\cO(|\cJ| \log |\cJ|)$ time.
For any $v \in V$, the sets $E_v, M_v, S_v, W_v, I_v, B_v, C_v$ can be computed in $\cO(|\cJ|)$ time, then the indices $e_v$, $a_y$ and $b_{v,y}$ can be computed in $\cO(|\cJ| \log |\cJ|)$ time (we may need to sort to compute the minimum and maximum values).
The DP has at most $\cO(n \cdot |\cJ|)$ states and an execution of \cref{alg:dp} at vertex $v$ takes $\cO(|\Ch(v)|)$ time (accounting for memoization), so the \cref{alg:dp} takes $\cO(n \cdot |\cJ| \cdot \sum_{v \in V} |\Ch(v)|) \subseteq \cO(n^2 \cdot |\cJ|)$ time.
Putting everything together, we see that the overall runtime is $\cO(|\cJ| \log |\cJ| + n^2 \cdot |\cJ|) \subseteq \cO(n^2 \cdot |\cJ|)$ since $|\cJ| \leq \binom{n}{2} \leq n^2$.
\end{proof}

\section{Experiments and implementation}
\label{sec:appendix-experiments}

The experiments are conducted on an Ubuntu server with two AMD EPYC 7532 CPU and 256GB DDR4 RAM.
Our code and entire experimental setup is available at\\ \url{https://github.com/cxjdavin/subset-verification-and-search-algorithms-for-causal-DAGs}.

\subsection{Implementation details}

\paragraph{Subset verification}
We implemented our subset verification algorithm and tested its correctness on random trees and random Erd\H{o}s-R\'{e}nyi graphs with random subsets of target edges $T$.
On random trees, we know that the subset verification should be 1 since intervening on the root always suffices regardless of what $T$ is.
On random Erd\H{o}s-R\'{e}nyi graphs $G^*$, we chose $T$ to be a random subset of covered edges of $G^*$ and checked that the subset verifying set is indeed a minimum vertex cover of $T$, whose size could be smaller than the full verification number $\nu(G^*)$.

\paragraph{Node-induced subset search}
We modified the 1/2-clique separator subroutine of \cite{gilbert1984separatorchordal} within the clique-separator based search algorithm of \cite{choo2022verification} by only assigning non-zero weights to endpoints of target edges.

\paragraph{Other full search algorithms that are benchmarked against}
We modified them to take in target edges $T$ and terminate early once all edges in $T$ have been oriented.

\subsubsection{Synthetic graph generation}

As justified by \cref{sec:sufficient-to-study-without-v-structures}, it suffices to study the performance of algorithms on connected DAGs without v-structures.
Our graphs are generated in the following way:
\begin{enumerate}
    \item Fix the number of nodes $n$ and edge probability $p$
    \item Generate a random tree on $n$ nodes
    \item Generate a random Erd\H{o}s-R\'{e}nyi graph $G(n,p)$
    \item Combine their edgesets and orient the edges in an acyclic fashion: orient $u \to v$ whenever vertex $u$ has a smaller vertex numbering than $v$.
    \item Add arcs to remove v-structures: for every v-structure $u \to v \gets w$ in the graph, we add the arc $u \to w$ whenever vertex $u$ has a smaller vertex numbering than $w$.
\end{enumerate}

\subsection{Experiment 1: Subset verification number for \emph{randomly chosen} target edges}

In this experiment, we study how the subset verification number scales when the target edges $T$ is chosen \emph{randomly}.

For each pair of graph parameters $(n,p)$, we generate 100 synthetic DAGs.
Then, for each graph with $|E| = m$ edges, we sampled a random subset $T \subseteq E$ of sizes $\{0.3m, 0.5m, 0.7m, m\}$ and ran our subset verification algorithm.
Additionally, we run the verification algorithm of \cite{choo2022verification} on the entire graph.
As expected, the verification number exactly matches the subset verification number in the special case where $|T| = m$.
We show these results in \cref{fig:plots}.
Despite the trend suggested in \cref{fig:plots}, the number of target edges is typically \emph{not} a good indication for the number of interventions needed to be performed and one can always construct examples where $|T'| > |T|$ but $\nu(G,T') \not> \nu(G,T)$.
For example, for a subset $T \subseteq E$, we have $\nu(G^*, T') = \nu(G^*, T)$ if $T' \supset T$ is obtained by adding edges that are already oriented by orienting $T$.
Instead, the number of ``independent target edges''\footnote{Akin to ``linearly independent vectors'' in linear algebra.} is a more appropriate measure.

\begin{figure}[htbp]
\centering
\begin{subfigure}[b]{0.4\linewidth}
    \centering
    \includegraphics[width=\linewidth]{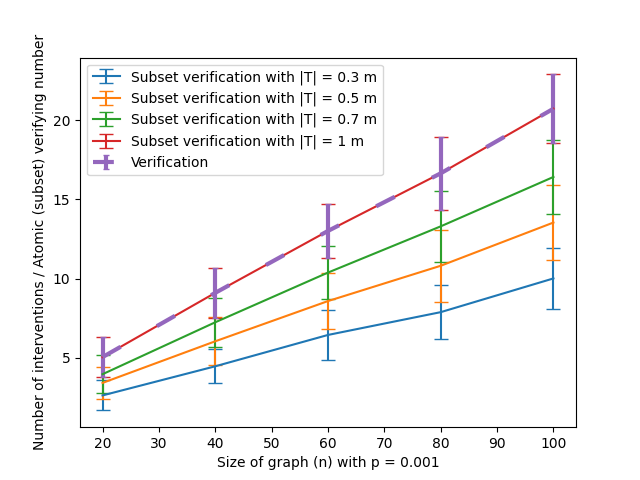}
    \caption{$p = 0.001$}
    \label{fig:p001}
\end{subfigure}
\begin{subfigure}[b]{0.4\linewidth}
    \centering
    \includegraphics[width=\linewidth]{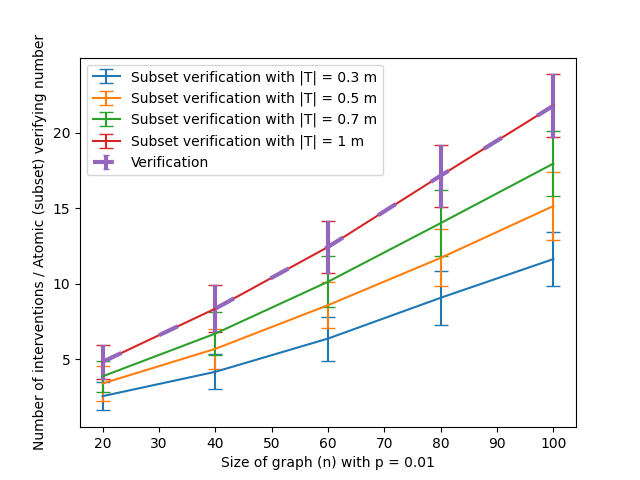}
    \caption{$p = 0.01$}
    \label{fig:p001}
\end{subfigure}
\begin{subfigure}[b]{0.4\linewidth}
    \centering
    \includegraphics[width=\linewidth]{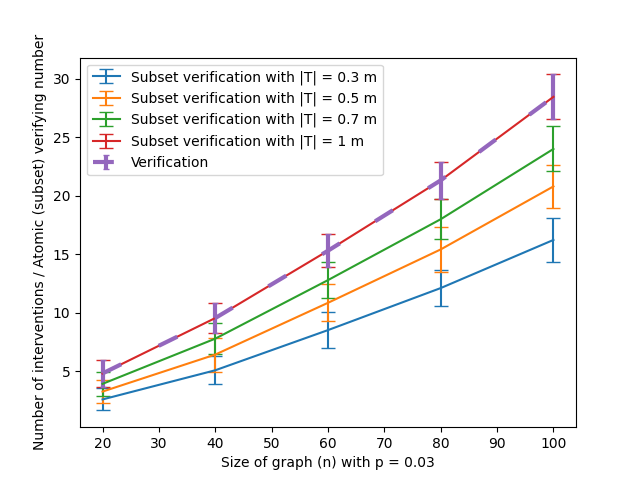}
    \caption{$p = 0.03$}
    \label{fig:p003}
\end{subfigure}
\begin{subfigure}[b]{0.4\linewidth}
    \centering
    \includegraphics[width=\linewidth]{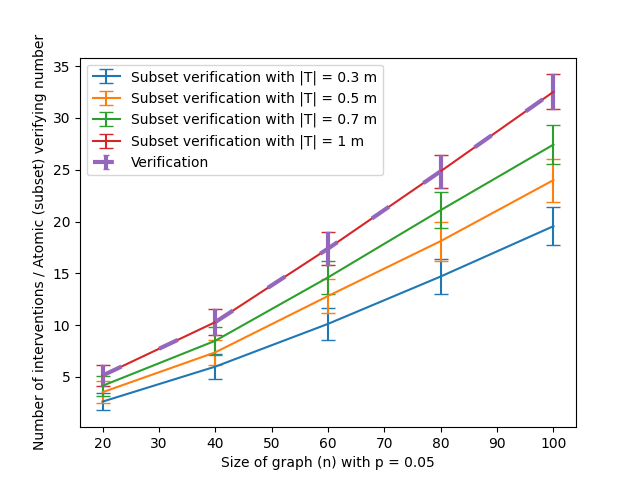}
    \caption{$p = 0.05$}
    \label{fig:p005}
\end{subfigure}
\begin{subfigure}[b]{0.4\linewidth}
    \centering
    \includegraphics[width=\linewidth]{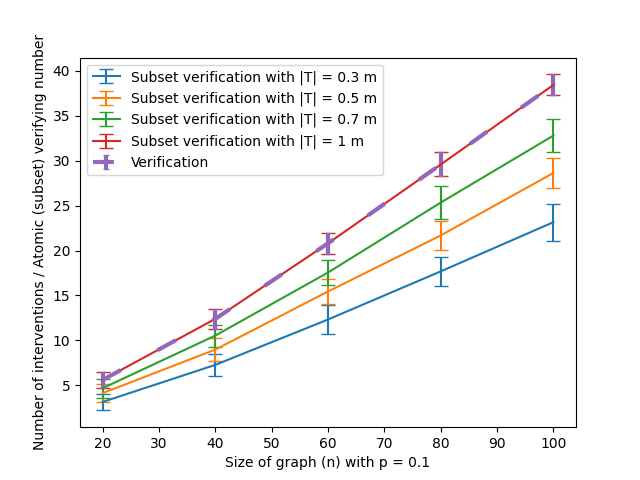}
    \caption{$p = 0.1$}
    \label{fig:p01}
\end{subfigure}
\begin{subfigure}[b]{0.4\linewidth}
    \centering
    \includegraphics[width=\linewidth]{plots/p=0.3.png}
    \caption{$p = 0.3$}
    \label{fig:p03}
\end{subfigure}
\caption{Plots for $p = \{0.001, 0.01, 0.03, 0.05, 0.1, 0.3\}$ across $n = \{10, 20, 30, \ldots, 100\}$. Observe that the subset verification number increases as the size of the random subset of target edges increases. Furthermore, in the special case of $|T| = m$, the subset verification number is exactly the verification number.}
\label{fig:plots}
\end{figure}

While the edge probability values may seem small, the graph is actually quite dense due to the addition of arcs to remove v-structures.
We show this in \cref{fig:density}, where we plot the number of edges of our generated graphs and compare it against the maximum number of possible edges.
Observe that the generated graph is almost a complete graph when $p = 0.3$.

\begin{figure}[htbp]
\centering
\includegraphics[width=0.45\textwidth]{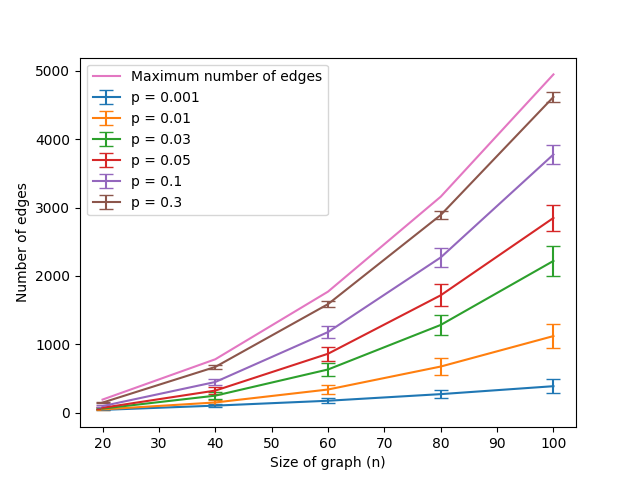}
\caption{
We plot the number of edges in our generated synthetic graphs and compare it against the maximum $\binom{n}{2}$ number of edges.
Observe that the generated graph is almost a complete graph when $p = 0.3$.
}
\label{fig:density}
\end{figure}

\subsection{Experiment 2: Local causal graph discovery}

In this experiment, we compare node-induced subset search with full search algorithms on the task of local causal graph discovery where we only wish to orient edges around a target node of interest.
Following \cite{choo2022verification}, we base our evaluation on the experimental framework of \cite{squires2020active} which empirically compares atomic intervention policies.

We compared the following atomic intervention algorithms against the atomic verification number $\nu_1(G^*)$ and atomic subset verification number $\nu_1(G^*, T)$; see \cref{fig:r-hops}:
\begin{description}
    \item[\texttt{random}:] A baseline algorithm that repeatedly picks a random non-dominated node (a node that is incident to some unoriented edge) from the interventional essential graph
    \item[\texttt{dct}:] \texttt{DCT Policy} of \cite{squires2020active}
    \item[\texttt{coloring}:] \texttt{Coloring} of \cite{shanmugam2015learning}
    \item[\texttt{separator}:] Clique-separator based search algorithm of \cite{choo2022verification}
    \item[\texttt{SubsetSearch}:] Our modification of \texttt{separator} that treats the union of endpoints of given target edges as the vertices in the node-induced subgraph of interest.
    That is, we may end up increasing the set of target edges $T \subseteq E$ if the input $T$ was not already all edges within a node-induced subgraph.
    However, note that the given inputs $T$ for this experiment already includes all edges within a node-induced subgraph so this is not a concern.
\end{description}

While our algorithms to construct the Hasse diagram and solve the produced interval stabbing problem is fast, we remark that the current implementation for computing $\{R(G^*,v)\}_{v \in V}$ in the \texttt{causaldag} package\footnote{\url{https://causaldag.readthedocs.io/en/latest/\#}} can be slow.
In particular, it is \emph{not} the $\cO(d \cdot |E|)$ time algorithm of \cite[Algorithm 2]{pmlr-v161-wienobst21a} mentioned in \cref{sec:appendix-meek-rules}.
In our experiments, computing $\{R(G^*,v)\}_{v \in V}$ takes up more than 98\% of the running time for computing subset verification numbers for each graph $G^*$.
However, note that in practical use case scenarios, one simply use the algorithms without actually needing computing $\{R(G^*,v)\}_{v \in V}$, so this is not a usability concern.

\begin{figure}[htbp]
\centering
\begin{subfigure}[b]{0.4\linewidth}
    \centering
    \includegraphics[width=\linewidth]{plots/p0001_hop1_interventioncount.png}
    \caption{$r=1$}
\end{subfigure}
\begin{subfigure}[b]{0.4\linewidth}
    \centering
    \includegraphics[width=\linewidth]{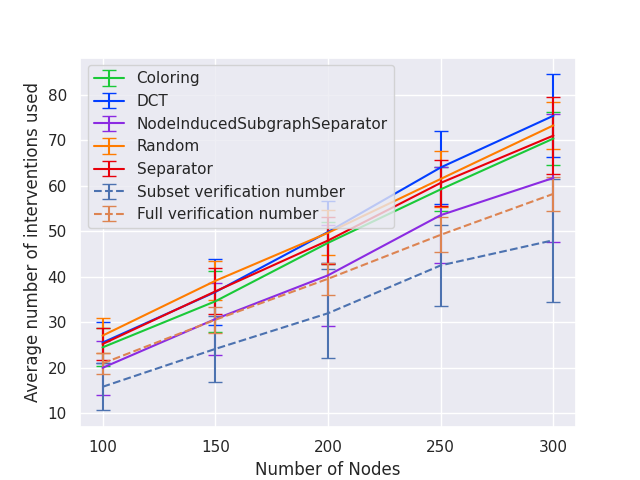}
    \caption{$r=3$}
\end{subfigure}
\caption{\texttt{SubsetSearch} consistently uses less interventions than existing state-of-the-art full graph search algorithms when we only wish to orient edges within a r-hop neighborhood of a randomly chosen target node $v$, for $r \in \{1,3\}$.}
\label{fig:r-hops}
\end{figure}

\end{document}